\documentclass{article}
\PassOptionsToPackage{numbers, sort&compress}{natbib}
\usepackage[main,final]{neurips_2026}

\usepackage{algorithm, algorithmic}

\usepackage[utf8]{inputenc} 
\usepackage[T1]{fontenc}    
\usepackage{hyperref}       
\usepackage{url}            
\usepackage{booktabs}       
\usepackage{amsfonts}       
\usepackage{nicefrac}       
\usepackage{microtype}      
\usepackage{xcolor}         

\usepackage{amsmath}
\usepackage{graphicx}
\usepackage{amssymb}
\usepackage{amsthm}

\newtheorem{theorem}{Theorem}

\newtheorem{proposition}[theorem]{Proposition}
\newtheorem{corollary}[theorem]{Corollary}

\theoremstyle{definition}

\theoremstyle{remark}
\newtheorem{remark}{Remark}
\title{Composing Diffusion Priors with Explicit Physical Context via Generative Gibbs Sampling}

\author{
    Weizhou Wang \\
  Department of Chemistry\\
  University of Chicago\\
  Chicago, IL 60637 \\
  \texttt{weizhouwang@uchicago.edu} \\
    \And
    Jonathan Weare \\
    Courant Institute of Mathematical Sciences \\
    New York University \\
    New York, NY 10012  \\
    \texttt{weare@nyu.edu} \\
    \And
    Aaron R. Dinner\\
  Department of Chemistry\\
  University of Chicago\\
  Chicago, IL 60637 \\
\texttt{dinner@uchicago.edu} \\
}

\begin{document}

\maketitle

\begin{abstract}
Pretrained diffusion models provide powerful learned priors, but in scientific sampling the target distribution often depends on physical context that a single generative model can neither fully represent nor learn from scarce joint-system data. We introduce Generative Gibbs for Physics-Aware Sampling (GG-PA), a training-free framework that composes learned partial priors with explicit physical context through Gibbs sampling. We prove that this sampler is asymptotically exact as the diffusion time vanishes and remains exact at finite diffusion time for quadratic interactions. As a special case, the finite-time result corrects a finite-noise bias in plug-and-play split Gibbs samplers for linear inverse problems at no additional cost. We further introduce replica exchange over diffusion time to accelerate mixing. In experiments, a one-dimensional double-well prior recovers a context-induced distribution shift when coupled to an environment, and copies of such a prior composed on a lattice reproduce the emergent collective behavior of the $\phi^4$ model. In atomistic peptide systems, GG-PA captures a context-induced distribution shift and the spontaneous formation of hydrogen-bonded dimers from two copies of a prior trained only on the monomer.
\end{abstract}

\section{Introduction}

Pretrained diffusion models are increasingly attractive as priors for scientific sampling \cite{song2020score,ho2020denoising,croitoru2023diffusion}. However, the target distribution of a system often depends strongly on its physical context. Solvents, ions, ligands, and external fields can all substantially reshape the equilibrium ensemble of the resulting joint system and give rise to collective behavior \cite{baldwin1996hofmeister,bekard2014electric,wicky2017affinity,huihui2018modulating,guo2024dynamics}. Anticipating every possible joint system is infeasible, and even the known cases are too many for a single generative model to cover. The number of such systems grows combinatorially, leading to data scarcity for training such a model \cite{laine2025navigating,sledzieski2026landscape}. Moreover, most pretrained models describe only a subset of the degrees of freedom, such as an isolated protein backbone or a molecular fragment. The effects of the environment and other molecules remain implicit, so these models generally fail once their context changes \cite{saldano2022impact,hu2024patch,gottschling2025troublesome}.

Inference-time approaches are then the only practical route. Instead of retraining or finetuning on context-labeled data \cite{guo2024diffusion}, they reuse a pretrained model and incorporate the context only during sampling, through guidance \cite{dhariwal2021diffusion,ho2022classifier}, posterior sampling \cite{song2021solving,chung2022diffusion}, or MCMC-based correction \cite{song2019generative,doucet2022score,wu2023practical,xun2025posterior}. These methods, however, can express the context only through the variables of the generative model. Degrees of freedom outside the prior must then be marginalized into an effective (free) energy \cite{chipot2007free}, which is often intractable for high-dimensional environments and redundant when those degrees of freedom are already described explicitly by pretrained priors or force fields \cite{hornak2006comparison,lindorff2010improved,pall2020heterogeneous,eastman2023openmm}. Thus the challenge we consider is not conditional generation under a single prior but inference-time composition of partial learned priors with explicit physical context.

We address this challenge with \textbf{Generative Gibbs for Physics-Aware Sampling (GG-PA)}, an approach for composing multiple pretrained diffusion priors with explicit physical context during inference. GG-PA maintains an explicit representation of the full system and couples it to the pretrained priors through projections onto the prior variables. Sampling alternates between two Gibbs updates. In a \emph{parallel denoising step}, each prior updates its variables conditioned on the current full-system state. In a \emph{context-aware aggregation step}, the full-system state updates under the explicit physical context and soft restraints to the prior variables. Each prior is thus trained once on its own component and can be reused unchanged across joint systems that contain it, without any training data from those systems.

GG-PA defines a family of joint distributions of the prior and full-system variables indexed by diffusion time $t$, the forward-diffusion noise level. In the limit $t \to 0$ without noise, we show that we recover the desired composed distribution, establishing asymptotic correctness under the stated factorization assumptions. We further show that GG-PA remains exact at finite diffusion times when the interactions between pretrained and/or physical models are quadratic. As a special case, this result corrects a finite-noise bias in plug-and-play split Gibbs samplers for linear inverse problems \cite{vono2019split,coeurdoux2024plug,wu2024principled}. Small $t$ enforces strong consistency with the pretrained priors but can make the system stiff and lead to poor mixing. Larger $t$ relaxes the coupling and improves exploration, but outside the exact regime it weakens fidelity to the priors. To balance these effects, we introduce replica exchange over diffusion time \cite{swendsen1986replica,sugita1999replica,shirts2008statistically,he2026crepe}, which admits a tractable swap criterion and lets high-noise replicas explore rapidly while low-noise replicas preserve fidelity.

We evaluate GG-PA on three systems of increasing complexity. A one-dimensional double-well prior coupled to an explicit environment recovers the environment-induced distribution shift, and replica exchange accelerates mixing at small $t$. Copies of such a prior composed on a two-dimensional (2D) lattice recover the emergent collective behavior of a Ginzburg-Landau $\phi^4$ model \cite{milchev1986finite} without any data from the joint system. In atomistic peptide systems, where molecular dynamics supplies the context \cite{tobias1992conformational} and interactions are no longer quadratic, GG-PA captures an ion-induced shift and composes two copies of a monomer prior into hydrogen-bonded dimers. In follow-up work \cite{wang2026nuclear}, the same construction extends to imaginary-time path integrals for simulating nuclear quantum effects, where a single prior trained on classical statistics transfers across temperatures and isotopes. For liquid water, the composed target reaches 62,208 dimensions.

Our contributions are threefold. First, we formulate composition of partial diffusion priors with physical context as inference over an explicit full-system state and the variables associated with the pretrained priors, bypassing the intractable marginalization required by standard posterior sampling. Second, we derive the GG-PA sampler and characterize its correctness through asymptotic recovery as $t \to 0$ and exact finite-$t$ sampling in settings with quadratic interactions, which also corrects a finite-noise bias in existing split Gibbs samplers. We further introduce replica exchange over diffusion time for improved mixing. Third, we demonstrate modular multi-prior composition numerically in systems with and without quadratic interactions.

\section{Related Work}

\subsection{Inference-time conditioning and composition for diffusion priors}
A growing body of work adapts pretrained diffusion priors at inference time through guidance \cite{song2020score,ho2022classifier,chung2022diffusion,liu2025exendiff,chenming2026}, score composition \cite{liu2022compositional}, posterior sampling \cite{song2021solving,chung2022diffusion}, and Monte Carlo corrections \cite{doucet2022score,wu2023practical,xun2025posterior,du2023reduce}. These methods typically express the new task as a posterior or composition over the variables of the prior(s), with additional information entering as a tractable likelihood or energy on those variables. Applying them when the priors cover only part of the system would require marginalizing the remaining explicit variables, which is generally intractable. By contrast, in GG-PA, pretrained priors cover only selected degrees of freedom, and additional variables and interactions remain explicit. Rather than requiring all context to be expressed through the prior variables, GG-PA composes one or more partial learned priors with explicit full-system context through a joint target on an augmented state space.

\subsection{Bayesian plug-and-play and split samplers}
At an algorithmic level, GG-PA is close to plug-and-play (PnP) and variable splitting methods that alternate between prior-driven and likelihood-dependent updates \cite{venkatakrishnan2013plug,chan2016plug,zhang2021plug,kamilov2023plug,buzzard2018plug,laumont2022bayesian,bouman2023generative,coeurdoux2024plug,wu2024principled}. The distinction is the target of inference. Bayesian PnP methods introduce an auxiliary variable associated with the likelihood to separate it from the prior. GG-PA instead targets a joint distribution over an explicit full-system state and the variables of pretrained priors, so the augmented state space is part of the composition rather than only a splitting device.

\section{Method}

\subsection{Augmented State Space and Joint Target}

We consider an explicit full-system state $\mathbf{s}\in\mathcal{S}$, such as an all-atom system including solvent. This state is weighted by a desired context factor $\pi_{\mathrm{ctx}}(\mathbf{s})$, and its projections are coupled to $K$ pretrained diffusion priors. The context factor is an evaluable, possibly unnormalized probabilistic factor over $\mathbf{s}$. In physical applications, it is typically a Boltzmann factor derived from an energy function.

For each $i=1,\dots,K$, let $\mathbf{x}_i\in\mathcal{X}_i$ denote the variable modeled by diffusion prior $i$ (e.g., a protein backbone or substructure), with implicit prior $p_i(\mathbf{x}_i)$ and training forward diffusion kernel $q_t^{(i)}(\mathbf{y}\mid \mathbf{x}_i)$, where $t\in[0,1]$ is the diffusion time. 

Each prior is connected to the explicit full-system state through a projection $\Phi_i:\mathcal{S}\to\mathcal{X}_i$. Crucially, the collection of projections need not represent all degrees of freedom of $\mathbf{s}$. Any variables not represented by the pretrained priors remain explicit and are governed directly by the context factor. 

The augmented state space is defined as $\mathcal{Z}=\mathcal{S}\times\prod_{i=1}^K\mathcal{X}_i$, and we define a family of joint target densities indexed by diffusion time $t$ in this space:
\begin{equation}
    \pi_t\big(\mathbf{s},\{\mathbf{x}_i\}\big)
    \propto
    q_{\mathrm{ctx}}(\mathbf{s},t)
    \prod_{i=1}^K
    \left[
        p_i(\mathbf{x}_i)
        \cdot
        q_t^{(i)}\big(\Phi_i(\mathbf{s})\mid \mathbf{x}_i\big)
    \right],
    \label{eq:joint_distribution_multi}
\end{equation}
where $q_{\mathrm{ctx}}(\mathbf{s},t)$ is the finite-$t$ context factor used in the joint target. The collection $\{q_{\mathrm{ctx}}(\cdot,t)\}_{t\in[0,1]}$ defines a context schedule satisfying $q_{\mathrm{ctx}}(\mathbf{s},0)=\pi_{\mathrm{ctx}}(\mathbf{s})$. Unless otherwise noted, we use the constant schedule $q_{\mathrm{ctx}}(\mathbf{s},t)=\pi_{\mathrm{ctx}}(\mathbf{s})$. We refer to nontrivial $t$-dependent choices as annealed context schedules, which are introduced only in cases that admit exact finite-$t$ behavior. The forward diffusion kernels act as couplings enforcing consistency between the explicit full-system state and the variables of the priors. Following the standard setup in diffusion models \cite{song2020score,ho2020denoising}, as $t\to 1$, these couplings broaden and the prior variables decouple from $\mathbf{s}$, whereas as $t\to 0$,
\begin{equation}
    \begin{aligned}
    \lim_{t\rightarrow 0} q_t^{(i)}\big(\Phi_i(\mathbf{s})\mid \mathbf{x}_i\big)
    &= \delta\!\big(\Phi_i(\mathbf{s})-\mathbf{x}_i\big),
    \\
    \pi_0\big(\mathbf{s},\{\mathbf{x}_i\}\big)
    \propto
    \pi_{\mathrm{ctx}}(\mathbf{s})
    &\prod_{i=1}^K
    \left[
        p_i(\mathbf{x}_i)\,
        \delta\!\big(\Phi_i(\mathbf{s})-\mathbf{x}_i\big)
    \right],
    \label{eq:t0_limit}
    \end{aligned}
\end{equation}
and marginalizing $\{\mathbf{x}_i\}$ yields $\pi_0(\mathbf{s})\propto \pi_{\mathrm{ctx}}(\mathbf{s}) \prod_{i=1}^K p_i\!\big(\Phi_i(\mathbf{s})\big)$, a product of experts~\cite{hinton2002training} over projections of a shared state. In this limit, the full-system state is strictly consistent with the high-density regions of the priors while the remaining variables are weighted by the desired context factor.

\subsection{Generative Gibbs Sampling Framework}

\begin{algorithm}[t]
\caption{Generative Gibbs for Physics-Aware Sampling (GG-PA)}
\label{alg:ggpa}
\begin{algorithmic}[1]
\REQUIRE Context schedule $q_{\mathrm{ctx}}(\cdot,t)$, pretrained diffusion priors $\{p_i\}_{i=1}^K$, diffusion time $t$, sweeps $N$
\STATE Initialize full-system state $\mathbf{s}^{(0)}$
\FOR{$n = 1, 2, \dots, N$}
    \FOR{$i = 1$ \TO $K$ \textbf{in parallel}}
        \STATE Sample $\mathbf{x}_i^{(n)} \sim p_i\!\left(\mathbf{x}_i \mid \Phi_i(\mathbf{s}^{(n-1)}), t\right)$
    \ENDFOR
    \STATE Sample $\mathbf{s}^{(n)}$ under $U_{\mathrm{eff}}(\mathbf{s}; \{\mathbf{x}_i^{(n)}\}, t)$
\ENDFOR
\RETURN $\{\mathbf{s}^{(n)}\}_{n=1}^N$
\end{algorithmic}
\end{algorithm}

To sample from the joint target in Eq.~\eqref{eq:joint_distribution_multi}, we employ an alternating Gibbs sampling scheme \cite{gelfand2000gibbs} that updates each variable based on its conditional distribution.

\paragraph{Step 1: Parallel Denoising Update.}
With $\mathbf{s}$ fixed, each prior variable $\mathbf{x}_i$ is updated independently by sampling from
\begin{equation}
    \pi_t(\mathbf{x}_i \mid \mathbf{s})
    \propto
    p_i(\mathbf{x}_i)\;
    q_t^{(i)}\!\big(\Phi_i(\mathbf{s}) \mid \mathbf{x}_i\big).
    \label{eq:denoising_multi}
\end{equation}
Treating $\mathbf{y}_i := \Phi_i(\mathbf{s})$ as a corrupted observation at diffusion level $t$, Eq.~\eqref{eq:denoising_multi} is the Bayesian denoising posterior induced by the prior $p_i$ and the forward diffusion kernel $q_t^{(i)}$ (see Appendix~\ref{sec:theoretical_details}). In practice, we run the pretrained model $i$'s reverse-time sampler initialized at $\mathbf{y}_i$ and integrate from time $t$ down to $0$ to obtain a sample $\mathbf{x}_i$ to approximate this conditional Gibbs update. The $K$ denoising updates across different priors can be performed in parallel.

\paragraph{Step 2: Context-Aware Aggregation Update.}
Fixing $\{\mathbf{x}_i\}$, we update the full-system state $\mathbf{s}$ from
\begin{equation}
  \pi_t(\mathbf{s} \mid \{\mathbf{x}_i\})
\propto
q_{\mathrm{ctx}}(\mathbf{s},t)\,
\prod_{i=1}^K q_t^{(i)}\!\big(\Phi_i(\mathbf{s}) \mid \mathbf{x}_i\big),  
\end{equation}
or equivalently under the effective potential
\begin{equation}
    U_{\mathrm{eff}}(\mathbf{s}; \{\mathbf{x}_i\}, t)
    =
    -\log q_{\mathrm{ctx}}(\mathbf{s},t)
    \;-\;
    \sum_{i=1}^K \log q_t^{(i)}\!\big(\Phi_i(\mathbf{s}) \mid \mathbf{x}_i\big),
    \label{eq:effective_energy_multi}
\end{equation}
so that $\mathbf{s}\sim\exp(-U_{\mathrm{eff}}(\mathbf{s}))$. Depending on the analytical nature of $q_{\mathrm{ctx}}$ and the projector $\Phi_i$, this step can be executed exactly within tractable families (e.g., linear-Gaussian) or numerically via standard MCMC methods such as Metropolis-Hastings \cite{metropolis1953equation} or Langevin dynamics.

\subsection{Theoretical Properties}

We next show that our framework is asymptotically exact for decomposable systems as $t\to 0$.
\begin{proposition}[Asymptotic Exactness for Decomposable Systems]
\label{prop:asymptotic_exactness}
Let the full-system state be decomposable as $\mathbf{s}=(\mathbf{s}_1,\dots,\mathbf{s}_K,\mathbf{s}_{\mathrm{env}})$, with true Boltzmann density $\pi(\mathbf{s})\propto \exp\!\left(-\beta V_{\mathrm{tot}}(\mathbf{s})\right)$ and total energy
\begin{equation}
    V_{\mathrm{tot}}(\mathbf{s})
    =
    \sum_{i=1}^K U_i(\mathbf{s}_i)
    +
    U_{\mathrm{env}}(\mathbf{s}_{\mathrm{env}})
    +
    U_{\mathrm{int}}(\mathbf{s}_1,\dots,\mathbf{s}_K,\mathbf{s}_{\mathrm{env}}).
    \label{eq:total_energy_decomposable_multi}
\end{equation}
Assume perfect diffusion priors for the isolated components, $p_i(\mathbf{x}_i)\propto \exp\!\left(-\beta U_i(\mathbf{x}_i)\right)$, projectors $\Phi_i(\mathbf{s})=\mathbf{s}_i$, and a context factor of the form
\begin{equation}
\pi_{\mathrm{ctx}}(\mathbf{s})
\propto
\exp\!\left(-\beta\left[
U_{\mathrm{env}}(\mathbf{s}_{\mathrm{env}})
+
U_{\mathrm{int}}(\mathbf{s}_1,\dots,\mathbf{s}_K,\mathbf{s}_{\mathrm{env}})
\right]\right).
\end{equation}
In other words, the component priors and the context factor define a non-overlapping decomposition of the total energy: the priors account only for the isolated component terms $U_i$, while $\pi_{\mathrm{ctx}}$ contributes only the remaining environmental and interaction terms.

Then, in the limit $t \to 0$, the marginal target exactly recovers the true physical distribution:
\begin{equation}
    \pi_0\big(\{\mathbf{x}_i\},\mathbf{s}_{\mathrm{env}}\big)
    \propto
    \exp\!\left(-\beta V_{\mathrm{tot}}(\mathbf{x}_1,\dots,\mathbf{x}_K,\mathbf{s}_{\mathrm{env}})\right).
    \label{eq:asymptotic_target}
\end{equation}
\end{proposition}
\begin{proof}
Substituting the chosen priors and context into Eq.~\eqref{eq:t0_limit} with $\Phi_i(\mathbf{s})=\mathbf{s}_i$ yields a factor $\delta(\mathbf{s}_i-\mathbf{x}_i)$ for all $i$ in the $t\to0$ limit. Marginalizing $\{\mathbf{s}_i\}_{i=1}^K$ then recovers the Boltzmann factor in Eq.~\eqref{eq:asymptotic_target}.
\end{proof}
Proposition~\ref{prop:asymptotic_exactness} assumes that each energy term enters exactly one factor. When this condition fails, Eq.~\eqref{eq:joint_distribution_multi} still defines a valid target, but shared terms are double-counted, and the $t\to0$ marginal generally differs from the desired Boltzmann distribution.

At finite diffusion time, the coupling kernels no longer enforce exact consistency between $\Phi_i(\mathbf{s})$ and $\mathbf{x}_i$, so the marginal target generally differs from the desired target. However, exact recovery remains possible at finite $t$ when the interactions are quadratic.

\begin{proposition}[Finite-Time Exactness via Gaussian Deconvolution]
\label{prop:finite_t_exactness}
Let $\mathbf{s}_{\mathrm{sys}}=(\mathbf{s}_1,\dots,\mathbf{s}_K)$ and $\mathbf{x}=(\mathbf{x}_1,\dots,\mathbf{x}_K)$. Suppose the interaction energy is quadratic in $\mathbf{s}_{\mathrm{sys}}$ given $\mathbf{s}_{\mathrm{env}}$, so that
\begin{equation}
\exp\!\big(-\beta U_{\mathrm{int}}(\mathbf{s}_{\mathrm{sys}},\mathbf{s}_{\mathrm{env}})\big)
\propto
\mathcal{N}\!\big(\mathbf{s}_{\mathrm{sys}}
\mid
\boldsymbol{\mu}_c(\mathbf{s}_{\mathrm{env}}),
\boldsymbol{\Sigma}_c(\mathbf{s}_{\mathrm{env}})
\big).
\label{eq:int_gaussian_cond_env}
\end{equation}
Assume further that the product of independent forward kernels is linear-Gaussian,
\begin{equation}
\prod_{i=1}^K q_t^{(i)}(\mathbf{s}_i\mid \mathbf{x}_i)
=
\mathcal{N}\!\big(\mathbf{s}_{\mathrm{sys}}\mid \mathbf{A}_t\mathbf{x},\boldsymbol{\Sigma}_t\big),
\end{equation}
and parameterize the interaction part of the finite-$t$ context schedule in the same family:
\begin{equation}
q_{\mathrm{int}}(\mathbf{s}_{\mathrm{sys}}\mid \mathbf{s}_{\mathrm{env}},t)
=
\mathcal{N}\!\big(\mathbf{s}_{\mathrm{sys}}
\mid
\boldsymbol{\mu}_q(\mathbf{s}_{\mathrm{env}},t),
\boldsymbol{\Sigma}_q(\mathbf{s}_{\mathrm{env}},t)
\big).
\label{eq:qint_gaussian_family}
\end{equation}
Then the finite-time marginal remains exactly uncorrupted,
\begin{equation}
\pi_t(\mathbf{x},\mathbf{s}_{\mathrm{env}})
\propto
\pi_0(\mathbf{x},\mathbf{s}_{\mathrm{env}}),
\end{equation}
if and only if the context moments satisfy
\begin{equation}
\begin{aligned}
\boldsymbol{\mu}_q(\mathbf{s}_{\mathrm{env}},t) &= \mathbf{A}_t \boldsymbol{\mu}_c(\mathbf{s}_{\mathrm{env}}), \\
\boldsymbol{\Sigma}_q(\mathbf{s}_{\mathrm{env}},t) &= \mathbf{A}_t\,\boldsymbol{\Sigma}_c(\mathbf{s}_{\mathrm{env}})\,\mathbf{A}_t^T - \boldsymbol{\Sigma}_t.
\end{aligned}
\label{eq:mixed_gaussian_matching}
\end{equation}
\end{proposition}

\begin{proof}[Proof Sketch]
Marginalizing $\mathbf{s}_{\mathrm{sys}}$ in the joint target yields the convolution of $q_{\mathrm{int}}(\mathbf{s}_{\mathrm{sys}}\mid \mathbf{s}_{\mathrm{env}},t)$ with $\mathcal{N}\!\big(\mathbf{s}_{\mathrm{sys}}\mid \mathbf{A}_t\mathbf{x},\boldsymbol{\Sigma}_t\big)$. Applying the Gaussian product-and-integral identity and matching the resulting mean and covariance to Eq.~\eqref{eq:int_gaussian_cond_env} gives Eq.~\eqref{eq:mixed_gaussian_matching}; full details are provided in Appendix~\ref{sec:proofs_propositions_corollary}.
\end{proof}

\begin{corollary}[Critical Diffusion Time Bound]
\label{cor:time_bound}
The construction in Proposition~\ref{prop:finite_t_exactness} is well-posed if and only if $\boldsymbol{\Sigma}_q$ is positive-semidefinite, equivalently
\begin{equation}
\boldsymbol{\Sigma}_t \preceq \mathbf{A}_t\,\boldsymbol{\Sigma}_c(\mathbf{s}_{\mathrm{env}})\,\mathbf{A}_t^T,
\label{eq:mixed_positive_condition}
\end{equation}
where $\preceq$ denotes positive-semidefinite ordering. This defines a maximum admissible diffusion time $t_{\max}$ beyond which exact finite-$t$ recovery is impossible. Intuitively, exact recovery is physically feasible only when the generative noise injected by the forward kernel ($\boldsymbol{\Sigma}_t$) does not exceed the intrinsic thermal uncertainty encoded in the quadratic interaction ($\boldsymbol{\Sigma}_c$).
\end{corollary}
\begin{remark}[Connection to split Gibbs sampling for linear inverse problems]
\label{rem:split_gibbs}
When $K=1$, $\Phi(\mathbf{s})=\mathbf{s}$, and the context factor is a linear-Gaussian likelihood $\mathcal{N}(\mathbf{y};\mathbf{H}\mathbf{x},\boldsymbol{\Sigma}_\eta)$ with forward operator $\mathbf{H}$, the framework reduces to the split Gibbs sampler used in plug-and-play posterior sampling \cite{vono2019split,coeurdoux2024plug,wu2024principled}. In this setting, Proposition~\ref{prop:finite_t_exactness} implies that one should choose the split likelihood covariance to be $\boldsymbol{\Sigma}_\eta - \rho^2 \mathbf{H}\mathbf{H}^\top$ rather than $\boldsymbol{\Sigma}_\eta$ to avoid a covariance inflation that  requires $\rho\to 0$ for exactness. Details and a numerical verification are provided in Appendices~\ref{sec:si_linear_inverse} and \ref{sec:si_gmm_verification}.
\end{remark}

\subsection{Diffusion-Time Replica Exchange}
\label{sec:replica_exchange}

By construction, small $t$ yields faithful but stiff targets, whereas large $t$ improves exploration at the cost of weaker prior consistency. To balance this trade-off, we perform replica exchange \cite{swendsen1986replica} across the diffusion-time axis.

We simulate $R$ parallel replicas $\{(\mathbf{s}^{(r)},\{\mathbf{x}_i^{(r)}\}, t_r)\}_{r=1}^R$ with $0<t_1<\dots<t_R<1$ and periodically propose swaps between neighboring replicas $r$ and $r+1$. Substituting the joint target density in Eq.~\eqref{eq:joint_distribution_multi} into the Metropolis-Hastings swap ratio causes the intractable priors $\prod_i p_i(\mathbf{x}_i)$ to cancel exactly, yielding the tractable acceptance rule
\begin{equation}
    \alpha
    =
    \min\!\left(
    1,\;
    \frac{w_r(t_{r+1})\, w_{r+1}(t_r)}{w_r(t_r)\, w_{r+1}(t_{r+1})}
    \right),
    \label{eq:re_accept_simplified_multi}
\end{equation}
where $w_r(t)=q_{\mathrm{ctx}}(\mathbf{s}^{(r)},t)\prod_{i=1}^K q_t^{(i)}\big(\Phi_i(\mathbf{s}^{(r)})\mid \mathbf{x}_i^{(r)}\big)$ combines the context factor and the coupling kernels of replica $r$ at diffusion time $t$.
The swap criterion therefore depends only on the forward diffusion kernels and the context schedule. For the constant schedule $q_{\mathrm{ctx}}(\mathbf{s},t)=\pi_{\mathrm{ctx}}(\mathbf{s})$, the context factors cancel from the swap ratio. The detailed algorithm with replica exchange (GG-PA-RE) is provided in Appendix~\ref{sec:theoretical_details}.

\section{Experiments}
We evaluate GG-PA on three systems of increasing complexity. The coupled double-well provides a visualizable quadratic-interaction benchmark that tests finite-$t$ exactness and the effectiveness of diffusion-time replica exchange. The $\phi^4$ lattice model tests whether the same composition mechanism scales to interacting many-body systems and preserves collective critical behavior. Finally, we test the ability to capture context-induced distribution shifts and zero-shot composition with atomistic models of peptides. Because we assume that no generative model has been trained on the joint systems, we use the uncomposed prior (Direct Diffusion) as a point of comparison. Data for the joint systems serve only as references.

\subsection{Coupled Double-Well System}
\label{sec:double_well}

We first consider a two-dimensional (2D) coupled double-well system with state $\mathbf{x}=(x_{\mathrm{sys}},x_{\mathrm{env}})\in\mathbb{R}^2$ and total potential energy (Fig.~\ref{fig:double_well}a):
\begin{equation}
    U_{\mathrm{tot}}(\mathbf{x}) = A(x_{\mathrm{sys}}^2 - 1)^2 + \frac{1}{2}k_c(x_{\mathrm{sys}} - x_{\mathrm{env}})^2 + \frac{1}{2}k_b(x_{\mathrm{env}} - u_{\mathrm{eq}})^2.
\end{equation}
We train a diffusion prior only on the isolated symmetric distribution of $x_{\mathrm{sys}}$ ($k_c=0$) and then compose it at sampling time with the explicit environmental context ($k_c>0$, $u_{\mathrm{eq}}=1$).
Because the interaction is quadratic and the forward diffusion kernel is Gaussian, this system satisfies the requirements for Proposition \ref{prop:finite_t_exactness} and Corollary~\ref{cor:time_bound}; the theoretical diffusion time bound is $t_{\max}=0.28$ (details in Appendix~\ref{sec:double_well_details}). We compare GG-PA with fixed $t$ (Algorithm \ref{alg:ggpa}), its diffusion-time replica-exchange variant (GG-PA-RE; Algorithm \ref{alg:ggpa_re}), and a prior-only baseline that ignores the explicit context (labeled Direct Diffusion). For the mixing analysis, we also report molecular dynamics (MD) governed by the target potential.

\begin{figure}[hbt]
    \centering
    \includegraphics[width=\linewidth]{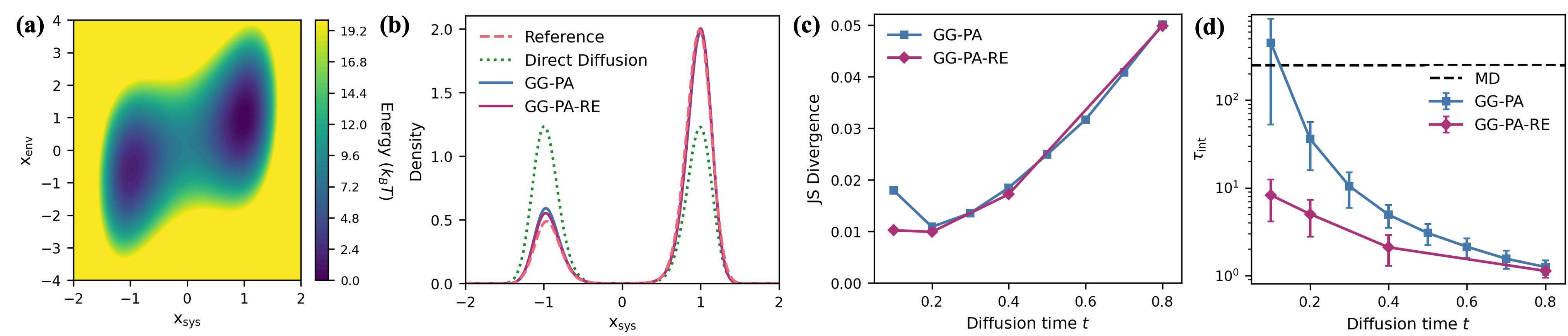}
    \caption{
    \textbf{Coupled double-well system.} 
    (a) The environment induces an asymmetric potential.
    (b) Marginal density of $x_{\mathrm{sys}}$. GG-PA (blue solid) and GG-PA-RE (purple solid) capture the asymmetry despite their symmetric prior (Direct Diffusion; green dotted). 
    (c) Jensen-Shannon divergence versus $t$. Below the diffusion time bound ($t \le 0.28$), GG-PA-RE remains consistently near the minimum observed error, whereas fixed-$t$ GG-PA rebounds at $t=0.1$ due to poor mixing. 
    (d) Integrated autocorrelation time $\tau_{\mathrm{int}}$. Replica exchange significantly accelerates mixing.
    }
    \label{fig:double_well}
\end{figure}
The results in Fig.~\ref{fig:double_well}b confirm that GG-PA successfully recovers the environment-induced asymmetry. We assess the performance as a function of diffusion time $t$ with the Jensen-Shannon (JS) divergence \cite{lin2002divergence} in Fig.~\ref{fig:double_well}c. 
The minimum error observed occurs at $t=0.2$ for both GG-PA and GG-PA-RE; given that the error is the same for the two algorithms, we attribute this error to that in the learned prior. The assumptions of Section \ref{sec:replica_exchange} (e.g., cancellation of $q_\mathrm{ctx}$) appear not to introduce additional error to GG-PA-RE.  For $t=0.1$, fixed-$t$ GG-PA exhibits markedly higher error than GG-PA-RE, which we attribute to poor mixing.  We confirm this hypothesis through the integrated autocorrelation time $\tau_{\mathrm{int}}$ in Fig.~\ref{fig:double_well}d. At $t=0.1$, fixed-$t$ GG-PA mixes more slowly than MD.  Otherwise, both algorithms outperform MD, with GG-PA-RE providing the biggest speedups in the stiff low-$t$ regime.
Returning to Fig.~\ref{fig:double_well}c, 
for diffusion times above the bound $t_{\max}$, the error grows steadily larger with $t$ for both fixed-$t$ GG-PA and GG-PA-RE. The quantitative agreement between the two algorithms again suggests a common source of error, in this case the breakdown of Proposition \ref{prop:finite_t_exactness} as predicted.

\subsection{$\phi^4$ Lattice Model}
\label{sec:phi4_lattice}

We next test whether the same sampling-time composition can scale to an interacting many-body system with collective behavior that is absent from the prior. To this end, we consider a 2D Ginzburg-Landau $\phi^4$ lattice field theory \cite{milchev1986finite}. The system is a scalar field $\phi\in\mathbb{R}^{L\times L}$ ($L=32$) with energy:
\begin{equation}
H(\phi)=\sum_i (\phi_i^2-1)^2+J\sum_{\langle i,j\rangle}(\phi_i-\phi_j)^2-h\sum_i\phi_i.
\label{eq:phi4_energy}
\end{equation}
Because the intersite interaction term is quadratic and the forward diffusion kernel is linear-Gaussian, this model again yields exact results at finite $t$. We first train a diffusion prior solely on the local on-site double-well factor $\exp[-(\phi_i^2-1)^2]$ and then compose $L\times L$ copies of it to sample the target distribution $\pi(\phi) \propto \exp(-H(\phi)/T)$.

In our evaluation, we fix the temperature at $T=1$ and scan the parameters $J$ and $h$. We measure the order parameter $m=(1/L^2)\sum_i \phi_i$, the susceptibility $\chi=L^2\left(\langle m^2\rangle-\langle |m|\rangle^2\right)$, and the integrated autocorrelation time $\tau_{\mathrm{int}}$ of $|m|$. As a reference, we use the standard checkerboard Metropolis Monte Carlo (MC) and also evaluate the replica-exchange variant GG-PA-RE.

\begin{figure}[hbt]
    \centering
    \includegraphics[width=\linewidth]{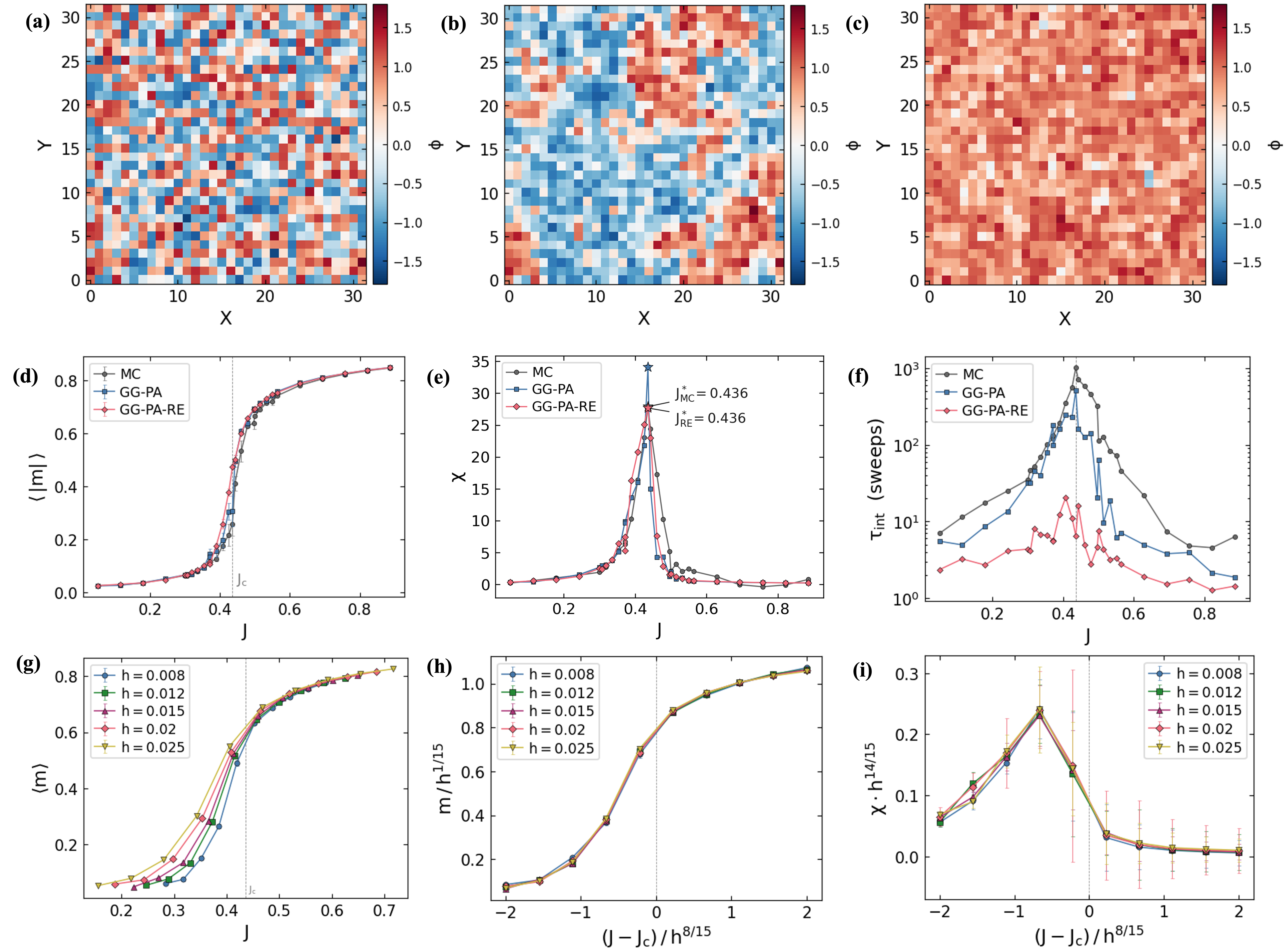}
    \caption{
    \textbf{2D Ginzburg-Landau $\phi^4$ model across the phase transition.}
    \textbf{(a--c):} Representative field configurations at $h=0$. 
    \textbf{(d--e):} Zero-field thermodynamic observables. GG-PA tracks the MC phase transition. \textbf{(f):} Integrated autocorrelation time $\tau_{\mathrm{int}}$. GG-PA-RE yields orders-of-magnitude speedups close to $J_c$ (dotted line).
    \textbf{(g--i):} Scaling and universal data collapse, confirming that GG-PA reproduces the expected critical behavior.
    }
    \label{fig:phi4_3by3}
\end{figure}

\paragraph{Equilibrium behavior and critical phenomena.}
As shown in Fig.~\ref{fig:phi4_3by3}, GG-PA correctly reproduces the equilibrium behavior of the model across the phase transition. Quantitatively, the thermodynamic observables agree closely with the MC reference across the full parameter scan. In particular, GG-PA captures the onset of spontaneous symmetry breaking, reflected in the sharp rise of $\langle |m| \rangle$, as well as the pronounced susceptibility peak at the critical coupling $J_c\approx 0.436$. Replica exchange further reduces the autocorrelation time near the critical point (Fig.~\ref{fig:phi4_3by3}f). Furthermore, in the presence of a finite external field $h$, the model belongs to the 2D Ising universality class, with critical exponents $\beta_{\mathrm{Ising}}=1/8$ and $\delta=15$ \cite{binney1992theory,pelissetto2002critical,goldenfeld2018lectures}. Applying the corresponding scaling transformation to GG-PA samples yields a collapse of both the scaled magnetization and susceptibility (Fig.~\ref{fig:phi4_3by3}h,i).

\paragraph{Physical mechanism of replica exchange.}
To understand the significant acceleration from replica exchange, we analyze the intermediate diffusion replicas (details in Appendix~\ref{sec:re_phase_transition}, Fig.~\ref{fig:re_phase}). Injecting generative noise softens the local double-well barriers, which analytically translates into a macroscopic attenuation of the effective intersite coupling. Consequently, the replica-dependent noise level acts as an additional thermodynamic control parameter, and high-noise replicas mix rapidly by accessing a disordered regime (Section \ref{sec:re_phase_transition}). This connects diffusion-time replica exchange to Hamiltonian parallel tempering \cite{fukunishi2002hamiltonian}: increasing diffusion time softens the effective landscape, allowing high-noise replicas to bypass barriers and accelerate mixing.

The success of this experiment is significant because the diffusion prior never sees intersite interactions during training. Therefore, the spontaneous emergence of complex collective behavior is not memorized by the prior but is driven entirely by our exact composition during inference. 

The same structure, independent copies of one learned single-site prior coupled by a quadratic interaction, arises in the path-integral representation of quantum statistics, where ring-polymer beads take the place of lattice sites \cite{chandler1981exploiting,wang2026quantum}. In follow-up work \cite{wang2026nuclear}, we use GG-PA to sample liquid water, composing $P=32$ copies of a prior over the 1,944 coordinates of 216 molecules into a 62,208-dimensional target. Unlike in the lattice model, the number of copies and the coupling strength between them correspond directly to the physical temperature and the isotopic mass of the target ensemble. A single prior can therefore sample ensembles with nuclear quantum effects across temperatures and isotopes without retraining.

\subsection{Alanine Dipeptide Systems}
\label{sec:alanine_systems}

Finally, we consider atomistic models of molecular systems based on  $N$-acetyl-alanyl-$N'$-methylamide (known as the alanine dipeptide \cite{tobias1992conformational}; henceforth abbreviated AD). Unlike the previous systems, the interactions are non-quadratic, and exact finite-$t$ recovery is no longer guaranteed. This setting therefore tests whether GG-PA remains effective beyond the quadratic / linear-Gaussian regime. In all cases, we start from a diffusion checkpoint trained exclusively on an isolated AD monomer in vacuum at $300$~K. We refer to this model as an isolated-monomer prior. To avoid equivariance complications, the prior is defined only on the backbone torsions $(\phi,\psi)$ of AD. We then compose this isolated-monomer prior at sampling time with explicit physical interactions evaluated by an MD engine \cite{eastman2023openmm}. Additional details are provided in Appendix~\ref{sec:ad_details}. For both systems, we report GG-PA-RE due to its faster mixing, using the full-system state $\mathbf{s}^{(1)}$ of its lowest-noise replica ($t=0.1$).

\begin{figure}[htpb]
    \centering
    \includegraphics[width=0.68\linewidth,trim={1bp 0 29bp 0},clip]{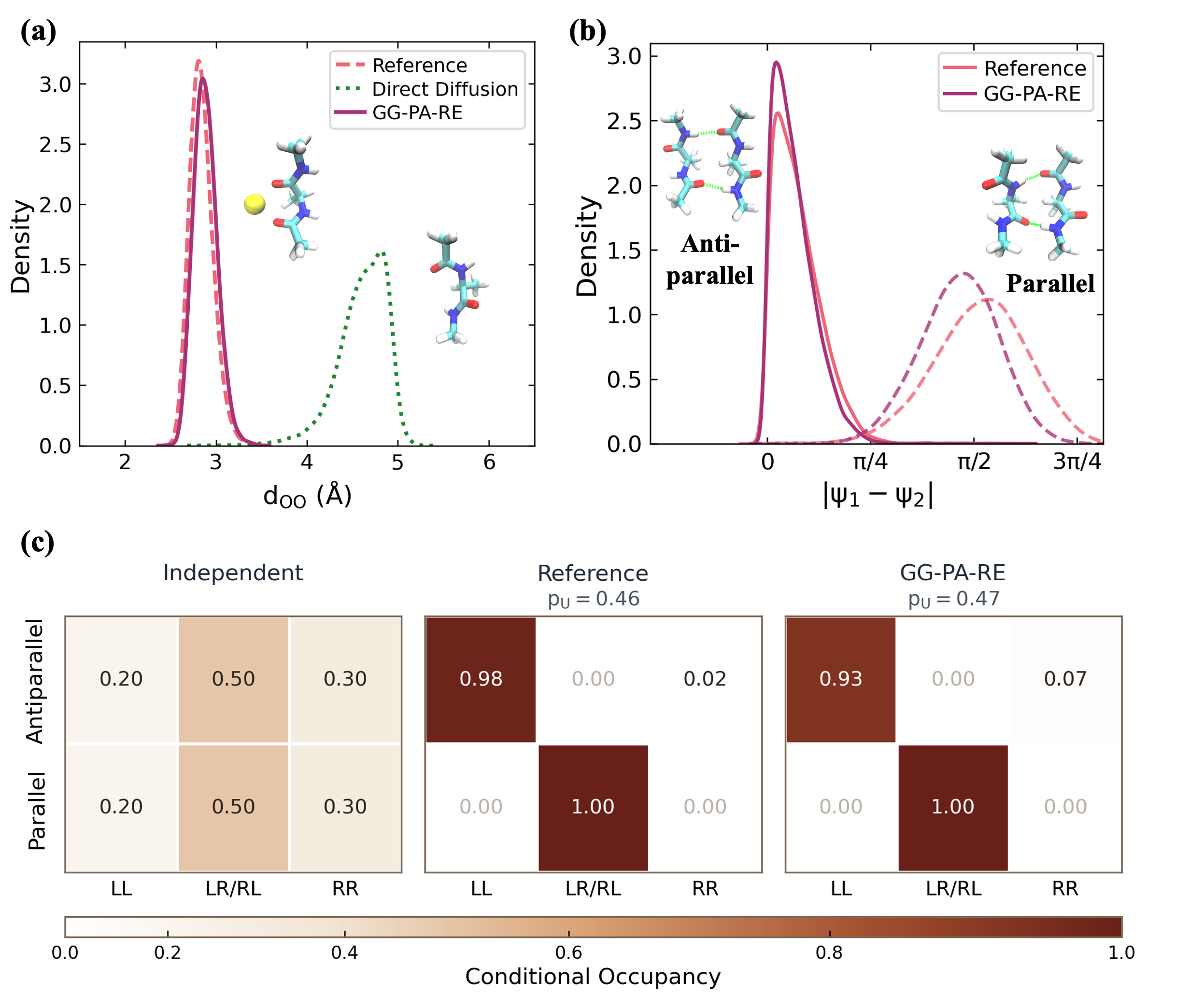} 
    \caption{\textbf{Alanine dipeptide systems.}
    \textbf{(a)} AD--Na$^+$: GG-PA-RE captures the distribution shift associated with ion coordination.
    \textbf{(b,c)} AD dimer: two copies of the isolated-monomer prior are composed to form hydrogen-bonded parallel and anti-parallel dimers. In \textbf{(c)}, the heatmaps show the occupancy of the dominant torsional states of the prior within each topology for two independent monomers (the product of the monomer state probabilities), the reference, and GG-PA-RE; $p_U$ denotes the residual probability. GG-PA-RE recovers the symmetry-broken organization absent from independent monomers.}
    \label{fig:ad_systems}
\end{figure}

\paragraph{AD--$\text{Na}^+$ System.}
We first couple the vacuum AD molecule with an explicit sodium cation ($\mathrm{Na}^+$). The carbonyl oxygens coordinate the cation, inducing a pronounced shift in the  O--O distance distribution. As shown in Fig.~\ref{fig:ad_systems}a, GG-PA-RE captures this shift with near quantitative accuracy, and representative structures look reasonable.

\paragraph{AD Dimer.}
Next, we consider AD dimers, which test composition of priors, like the $\phi^4$ lattice model example above. We use two copies of the pretrained isolated-monomer prior to model the backbone dihedrals of the two AD molecules, while the remaining physical interactions are governed by the context (i.e., obtained from MD simulations). 
In vacuum, the monomer prior has two dominant torsional states, denoted $L$ and $R$, located approximately at $(\phi,\psi)\approx(-139^\circ,154^\circ)$ and $(\phi,\psi)\approx(-82^\circ,63^\circ)$, respectively (details in Appendix~\ref{sec:ad_details}). When two AD molecules interact, intermolecular hydrogen bonding stabilizes anti-parallel and parallel topologies. For analysis, we first assign each dimer configuration a topology and then assign each monomer to its $L$ or $R$ state. This yields six assigned topology--torsion state categories,
\[
\{\textit{anti-parallel}, \textit{parallel}\}\times\{LL, LR/RL, RR\},
\]
with $LR$ and $RL$ merged since the two copies are identical, while configurations that do not satisfy the assignment criteria are collected into a residual unassigned category $U$. The precise assignment criteria are given in Appendix~\ref{subsec:ad_dimer_system}. If the two monomers were independent, the occupancy of each combination would be the product of the monomer state probabilities and would be the same in both topologies. In Fig.~\ref{fig:ad_systems}c, the heatmaps compare this independent case with the reference and GG-PA-RE, reporting the occupancy within each topology, while $p_U$ denotes the probability of the residual unassigned category.

As shown in Fig.~\ref{fig:ad_systems}c, both the reference and GG-PA-RE depart strongly from independent monomers and concentrate on selected topology--torsion state combinations. In particular, the anti-parallel topology is dominated by the $LL$ state, while the parallel topology selects the $LR/RL$ state. The dimer therefore does not inherit two independent monomer preferences but is strongly affected by intermolecular topology, and GG-PA-RE captures this emergent collective behavior from two copies of the monomer prior.

These results also validate diffusion-time replica exchange. Compared with fixed-$t$ GG-PA, GG-PA-RE switches between topologies orders of magnitude faster while retaining the high fidelity of small $t$ (Fig.~\ref{fig:si_addimer_iat}). On the same device, one switch between the anti-parallel and parallel topologies costs about 17 CPU seconds with GG-PA-RE, including all replicas, against about 120 s with MD. We expect this advantage to grow with barrier height because high-noise replicas in GG-PA-RE should be able to jump between states directly by leveraging the global information of the learned prior.

Unlike the systems with quadratic interactions above, the alanine dipeptide systems exhibit small residual deviations from the reference, mainly in the dihedral difference distribution of the dimer and the balance between the two topologies (Fig.~\ref{fig:ad_systems}b and Appendix~\ref{sec:si_ad_dimer_results}). We attribute the errors to three factors: (1) model capacity limits and approximation error in the base diffusion prior; (2) double counting of interactions from training the prior on data from MD simulations with the molecule fully represented by the force field; and (3) finite-$t$ approximation error, since the non-quadratic MD context makes the method only asymptotically exact as $t\to 0$. Despite these limitations, the results show that GG-PA-RE retains clear practical value for sampling both environmental effects and states of systems composed from multiple diffusion priors.

\section{Discussion and Conclusion}

Here we introduced Generative Gibbs for Physics-Aware Sampling (GG-PA), a framework for composing pretrained diffusion priors with explicit context during inference. Instead of requiring all relevant context to be expressed on the variables of a single diffusion prior, GG-PA treats diffusion priors, physical models, constraints, and simulators in a unified fashion through an augmented space. This enables a modular strategy in which learned priors and explicit physics are applied where they are most appropriate, which is especially useful in scientific systems with environmental degrees of freedom that are high-dimensional and/or readily handled by force fields.

Our results support three main conclusions. First, up to errors in the priors, GG-PA is asymptotically exact for decomposable systems and remains exact at finite diffusion times in systems with quadratic interactions. This finite-$t$ result can also be used to correct a covariance bias in existing split Gibbs samplers for linear inverse problems (Remark~\ref{rem:split_gibbs}). Second, diffusion-time replica exchange provides a practical way to mitigate the exploration-fidelity trade-off in stiff low-$t$ regimes. Third, beyond the quadratic / linear-Gaussian regime, GG-PA remains useful in atomistic systems and can recover context-induced distribution shifts and capture emergent collective behavior from isolated-component priors. Overall, these findings position GG-PA as a practical approach for combining pretrained generative priors with explicit physical contexts.

\paragraph{Limitations and Future Work.}
GG-PA has three main limitations. Exact finite-$t$ recovery currently relies on quadratic / linear-Gaussian structure, so other settings incur approximation error at finite $t$. As an MCMC-based method, GG-PA can also mix slowly in stiff or high-dimensional systems; replica exchange to improve mixing comes at the cost of carrying multiple replicas. Finally, performance depends on both prior fidelity and decomposition quality: poor projection design, low prior--context overlap, or residual double-counting can degrade results. Priors and context should share consistent thermodynamic conditions, and the restraint energy in the aggregation step indicates which prior disagrees with the context. Future work should therefore target annealed context schedules beyond the quadratic / linear-Gaussian setting, more efficient aggregation and tempering, multiscale decompositions \cite{cotler2023renormalizing,masuki2025generative}, priors for explicit force-field contexts, targeted refinement of priors \cite{zhang2026restrained}, and systematic cost comparisons with enhanced-sampling methods such as replica-exchange MD \cite{sugita1999replica}. More broadly, GG-PA suggests a modular paradigm in which reusable component-level priors are composed at inference time with explicit physical context, avoiding the need to retrain monolithic models whenever the system or environment changes. Natural next steps include coupling pretrained protein backbone priors \cite{wu2024protein} with explicit solvent or ion environments and assembling molecular complexes from fragment-level generative models.

\section*{Data Availability}
Code and scripts to reproduce all experiments, along with trained checkpoints and processed datasets,
are available at \url{https://github.com/weizhouwangw/gg-pa.git}. Raw trajectories and figure data are available upon request and will also be included in the repository release.

\begin{ack}
This work was supported by National Institutes of Health award R35 GM136381 and completed with computational resources administered by the University of Chicago Research Computing Center, including Beagle-3, a shared GPU cluster for biomolecular sciences supported by the NIH under the High-End Instrumentation (HEI) grant program award 1S10OD028655-0. JW's effort was supported by National Science Foundation award 2425899 and by the Margot and Tom Pritzker Science Foundation. The authors declare no competing interests.
\end{ack}

\bibliographystyle{unsrt}
\bibliography{references}

\newpage
\appendix
\renewcommand{\thefigure}{S\arabic{figure}}
\setcounter{figure}{0}
\renewcommand{\thetable}{S\arabic{table}}
\setcounter{table}{0}
\renewcommand{\theequation}{S\arabic{equation}}
\setcounter{equation}{0}
\renewcommand{\thealgorithm}{S\arabic{algorithm}}
\setcounter{algorithm}{0}

\renewcommand{\theHfigure}{S\arabic{figure}}
\renewcommand{\theHtable}{S\arabic{table}}
\renewcommand{\theHequation}{S\arabic{equation}}
\renewcommand{\theHalgorithm}{S\arabic{algorithm}}
\renewcommand{\theHtheorem}{S\arabic{theorem}}

\label{appendix:details}
\section{Theoretical Details}
\label{sec:theoretical_details}

This appendix collects the technical derivations underlying the claims in the main text. We first recall the posterior-sampling interpretation of diffusion denoising, which justifies the prior update in GG-PA. We then prove the finite-time exactness result in the setting with quadratic interactions and derive the corresponding critical diffusion-time condition. Finally, we derive the replica-exchange acceptance ratio and summarize the reduced potentials used for multi-state reweighting.

\subsection{Diffusion models as posterior samplers}
\label{si:diffusion_posterior}

The denoising step of a diffusion model can be interpreted as posterior sampling under a known corruption process. Let $\mathbf{x}(0)$ denote a clean sample and $\mathbf{x}(t)$ its corrupted version at diffusion time $t$. By Bayes' rule,
\begin{equation}
p(\mathbf{x}(0) \mid \mathbf{x}(t))
=
\frac{p(\mathbf{x}(0))\,q_t(\mathbf{x}(t) \mid \mathbf{x}(0))}
{\int p(\mathbf{x}')\,q_t(\mathbf{x}(t) \mid \mathbf{x}')\,d\mathbf{x}'}
\propto
p(\mathbf{x}(0))\,q_t(\mathbf{x}(t) \mid \mathbf{x}(0)),
\label{eq:si_bayes_posterior}
\end{equation}
where $p(\mathbf{x}(0))$ is the clean-data distribution and $q_t(\mathbf{x}(t) \mid \mathbf{x}(0))$ is the forward diffusion kernel. Thus, denoising amounts to sampling from the posterior distribution over clean states consistent with the observed noisy input $\mathbf{x}(t)$.

In our setting, the projected system state $\mathbf{y}_i := \Phi_i(\mathbf{s})$ plays the role of $\mathbf{x}(t)$ for prior $i$. The corresponding conditional factor in the joint target is
\begin{equation}
p_i(\mathbf{x}_i)\,q_t^{(i)}(\mathbf{y}_i \mid \mathbf{x}_i),
\end{equation}
up to normalization, so sampling $\mathbf{x}_i$ conditional on $\mathbf{y}_i$ is precisely the posterior-sampling problem in Eq.~\eqref{eq:si_bayes_posterior}. Operationally, a pretrained diffusion model provides an efficient approximate sampler for this conditional by initializing the reverse process at $\mathbf{y}_i$ and integrating the learned reverse-time dynamics from time $t$ to $0$ to obtain a sample $\mathbf{x}_i$ \cite{vincent2008extracting,ho2020denoising}. The same posterior-sampling interpretation also applies to the wrapped-Gaussian torsion diffusion used for alanine dipeptide, with $q_t$ replaced by the corresponding periodic forward kernel on the torus.

\subsection{Finite-time exactness in the Gaussian setting}
\label{sec:proofs_propositions_corollary}

We next prove Proposition~2 and derive Corollary~3 from the main text. In all Gaussian-kernel applications considered in this work, the stacked forward operator has the form $\mathbf{A}_t=\alpha_t I$ with $\alpha_t>0$ for the diffusion times of interest; we state the result for a general invertible matrix $\mathbf{A}_t$.
\setcounter{theorem}{1}
\begin{proposition}[Finite-Time Exactness via Gaussian Deconvolution]
Let $\mathbf{s}_{\mathrm{sys}}=(\mathbf{s}_1,\dots,\mathbf{s}_K)$ and $\mathbf{x}=(\mathbf{x}_1,\dots,\mathbf{x}_K)$. Suppose the interaction energy is quadratic in $\mathbf{s}_{\mathrm{sys}}$ conditional on $\mathbf{s}_{\mathrm{env}}$, so that
\begin{equation}
\exp\!\big(-\beta U_{\mathrm{int}}(\mathbf{s}_{\mathrm{sys}},\mathbf{s}_{\mathrm{env}})\big)
\propto
\mathcal{N}\!\big(
\mathbf{s}_{\mathrm{sys}}
\mid
\boldsymbol{\mu}_c(\mathbf{s}_{\mathrm{env}}),
\boldsymbol{\Sigma}_c(\mathbf{s}_{\mathrm{env}})
\big).
\label{eq:prop2_quad_interaction_si}
\end{equation}
Assume further that the stacked forward kernel is linear-Gaussian,
\begin{equation}
\prod_{i=1}^K q_t^{(i)}(\mathbf{s}_i \mid \mathbf{x}_i)
=
\mathcal{N}\!\big(
\mathbf{s}_{\mathrm{sys}}
\mid
\mathbf{A}_t \mathbf{x},
\boldsymbol{\Sigma}_t
\big),
\label{eq:prop2_forward_si}
\end{equation}
where $\mathbf{A}_t$ is square and invertible for the diffusion times under consideration, and parameterize the interaction part of the finite-$t$ context schedule in the same family:
\begin{equation}
q_{\mathrm{int}}(\mathbf{s}_{\mathrm{sys}}\mid \mathbf{s}_{\mathrm{env}},t)
=
\mathcal{N}\!\big(
\mathbf{s}_{\mathrm{sys}}
\mid
\boldsymbol{\mu}_q(\mathbf{s}_{\mathrm{env}},t),
\boldsymbol{\Sigma}_q(\mathbf{s}_{\mathrm{env}},t)
\big).
\label{eq:prop2_qint_si}
\end{equation}
Then the finite-time marginal remains exactly uncorrupted,
\begin{equation}
\pi_t(\mathbf{x},\mathbf{s}_{\mathrm{env}})
\propto
\pi_0(\mathbf{x},\mathbf{s}_{\mathrm{env}}),
\label{eq:prop2_goal_si}
\end{equation}
if and only if
\begin{equation}
\boldsymbol{\mu}_q(\mathbf{s}_{\mathrm{env}},t)
=
\mathbf{A}_t \boldsymbol{\mu}_c(\mathbf{s}_{\mathrm{env}}),
\qquad
\boldsymbol{\Sigma}_q(\mathbf{s}_{\mathrm{env}},t)
=
\mathbf{A}_t \boldsymbol{\Sigma}_c(\mathbf{s}_{\mathrm{env}})\mathbf{A}_t^\top
-
\boldsymbol{\Sigma}_t.
\label{eq:prop2_matching_si}
\end{equation}
\end{proposition}

\begin{proof}
We work in the decomposable setting of Proposition~1, now at finite diffusion time $t>0$. Decompose the full state as
\(
\mathbf{s}=(\mathbf{s}_{\mathrm{sys}},\mathbf{s}_{\mathrm{env}})
\),
where
\(
\mathbf{s}_{\mathrm{sys}}=(\mathbf{s}_1,\dots,\mathbf{s}_K)
\).
We choose the finite-$t$ context schedule in the factorized form
\begin{equation}
q_{\mathrm{ctx}}(\mathbf{s},t)
\propto
\exp\!\big(-\beta U_{\mathrm{env}}(\mathbf{s}_{\mathrm{env}})\big)\,
q_{\mathrm{int}}(\mathbf{s}_{\mathrm{sys}}\mid \mathbf{s}_{\mathrm{env}},t),
\label{eq:prop2_qctx_factorized_si}
\end{equation}
so that only the interaction term is annealed, while the environment contribution remains exact.

Under the joint target in the main text, the joint density is
\begin{equation}
\pi_t(\mathbf{s},\mathbf{x})
\propto
\exp\!\big(-\beta U_{\mathrm{env}}(\mathbf{s}_{\mathrm{env}})\big)\,
q_{\mathrm{int}}(\mathbf{s}_{\mathrm{sys}}\mid \mathbf{s}_{\mathrm{env}},t)
\prod_{i=1}^K
p_i(\mathbf{x}_i)\,
q_t^{(i)}(\mathbf{s}_i\mid \mathbf{x}_i).
\label{eq:prop2_joint_si}
\end{equation}
Using the stacked Gaussian forward model in Eq.~\eqref{eq:prop2_forward_si}, this can be rewritten as
\begin{equation}
\pi_t(\mathbf{s},\mathbf{x})
\propto
\exp\!\big(-\beta U_{\mathrm{env}}(\mathbf{s}_{\mathrm{env}})\big)
\Big(\prod_{i=1}^K p_i(\mathbf{x}_i)\Big)
q_{\mathrm{int}}(\mathbf{s}_{\mathrm{sys}}\mid \mathbf{s}_{\mathrm{env}},t)\,
\mathcal{N}\!\big(
\mathbf{s}_{\mathrm{sys}}
\mid
\mathbf{A}_t\mathbf{x},
\boldsymbol{\Sigma}_t
\big).
\label{eq:prop2_joint_gaussian_si}
\end{equation}

Marginalizing out $\mathbf{s}_{\mathrm{sys}}$ gives
\begin{equation}
\begin{aligned}
\pi_t(\mathbf{x},\mathbf{s}_{\mathrm{env}})
&\propto
\exp\!\big(-\beta U_{\mathrm{env}}(\mathbf{s}_{\mathrm{env}})\big)
\Big(\prod_{i=1}^K p_i(\mathbf{x}_i)\Big)
\\
&\qquad \times
\int d\mathbf{s}_{\mathrm{sys}}\;
q_{\mathrm{int}}(\mathbf{s}_{\mathrm{sys}}\mid \mathbf{s}_{\mathrm{env}},t)\,
\mathcal{N}\!\big(
\mathbf{s}_{\mathrm{sys}}
\mid
\mathbf{A}_t\mathbf{x},
\boldsymbol{\Sigma}_t
\big).
\end{aligned}
\label{eq:prop2_marginal_xt_si}
\end{equation}

By Proposition~1, the $t\to 0$ limit recovers
\begin{equation}
\pi_0(\mathbf{x},\mathbf{s}_{\mathrm{env}})
\propto
\exp\!\big(-\beta U_{\mathrm{env}}(\mathbf{s}_{\mathrm{env}})\big)
\Big(\prod_{i=1}^K p_i(\mathbf{x}_i)\Big)
\exp\!\big(-\beta U_{\mathrm{int}}(\mathbf{x},\mathbf{s}_{\mathrm{env}})\big).
\label{eq:prop2_pi0_form_si}
\end{equation}
Under the quadratic assumption in Eq.~\eqref{eq:prop2_quad_interaction_si}, this is equivalently
\begin{equation}
\pi_0(\mathbf{x},\mathbf{s}_{\mathrm{env}})
\propto
\exp\!\big(-\beta U_{\mathrm{env}}(\mathbf{s}_{\mathrm{env}})\big)
\Big(\prod_{i=1}^K p_i(\mathbf{x}_i)\Big)
\mathcal{N}\!\big(
\mathbf{x}
\mid
\boldsymbol{\mu}_c(\mathbf{s}_{\mathrm{env}}),
\boldsymbol{\Sigma}_c(\mathbf{s}_{\mathrm{env}})
\big).
\label{eq:prop2_pi0_gaussian_si}
\end{equation}

Comparing Eqs.~\eqref{eq:prop2_marginal_xt_si} and \eqref{eq:prop2_pi0_gaussian_si}, we see that
\(
\pi_t(\mathbf{x},\mathbf{s}_{\mathrm{env}})\propto \pi_0(\mathbf{x},\mathbf{s}_{\mathrm{env}})
\)
holds if and only if
\begin{equation}
\int d\mathbf{s}_{\mathrm{sys}}\;
q_{\mathrm{int}}(\mathbf{s}_{\mathrm{sys}}\mid \mathbf{s}_{\mathrm{env}},t)\,
\mathcal{N}\!\big(
\mathbf{s}_{\mathrm{sys}}
\mid
\mathbf{A}_t\mathbf{x},
\boldsymbol{\Sigma}_t
\big)
\propto
\mathcal{N}\!\big(
\mathbf{x}
\mid
\boldsymbol{\mu}_c(\mathbf{s}_{\mathrm{env}}),
\boldsymbol{\Sigma}_c(\mathbf{s}_{\mathrm{env}})
\big),
\label{eq:prop2_deconv_condition_si}
\end{equation}
where the proportionality constant may depend on $\mathbf{s}_{\mathrm{env}}$ and $t$, but not on $\mathbf{x}$.

Substituting the Gaussian ansatz in Eq.~\eqref{eq:prop2_qint_si}, the integral in Eq.~\eqref{eq:prop2_deconv_condition_si} becomes a Gaussian convolution:
\begin{equation}
\int d\mathbf{s}_{\mathrm{sys}}\;
\mathcal{N}\!\big(
\mathbf{s}_{\mathrm{sys}}
\mid
\boldsymbol{\mu}_q,
\boldsymbol{\Sigma}_q
\big)\,
\mathcal{N}\!\big(
\mathbf{s}_{\mathrm{sys}}
\mid
\mathbf{A}_t\mathbf{x},
\boldsymbol{\Sigma}_t
\big)
\propto
\mathcal{N}\!\big(
\mathbf{A}_t\mathbf{x}
\mid
\boldsymbol{\mu}_q,
\boldsymbol{\Sigma}_q+\boldsymbol{\Sigma}_t
\big),
\label{eq:prop2_gaussian_conv_si}
\end{equation}
where we suppress the dependence of $\boldsymbol{\mu}_q$ and $\boldsymbol{\Sigma}_q$ on $\mathbf{s}_{\mathrm{env}}$ and $t$ for notational simplicity.

Because $\mathbf{A}_t$ is invertible, the right-hand side of Eq.~\eqref{eq:prop2_gaussian_conv_si} can be rewritten as a Gaussian density in $\mathbf{x}$:
\begin{equation}
\mathcal{N}\!\big(
\mathbf{A}_t\mathbf{x}
\mid
\boldsymbol{\mu}_q,
\boldsymbol{\Sigma}_q+\boldsymbol{\Sigma}_t
\big)
=
|\det \mathbf{A}_t|^{-1}
\mathcal{N}\!\big(
\mathbf{x}
\mid
\mathbf{A}_t^{-1}\boldsymbol{\mu}_q,
\mathbf{A}_t^{-1}(\boldsymbol{\Sigma}_q+\boldsymbol{\Sigma}_t)\mathbf{A}_t^{-\top}
\big).
\label{eq:prop2_pullback_gaussian_si}
\end{equation}
Therefore, exact finite-time matching in Eq.~\eqref{eq:prop2_deconv_condition_si} is equivalent to
\begin{equation}
\mathbf{A}_t^{-1}\boldsymbol{\mu}_q(\mathbf{s}_{\mathrm{env}},t)
=
\boldsymbol{\mu}_c(\mathbf{s}_{\mathrm{env}}),
\label{eq:prop2_mean_match_inverse_si}
\end{equation}
and
\begin{equation}
\mathbf{A}_t^{-1}\big(\boldsymbol{\Sigma}_q(\mathbf{s}_{\mathrm{env}},t)+\boldsymbol{\Sigma}_t\big)\mathbf{A}_t^{-\top}
=
\boldsymbol{\Sigma}_c(\mathbf{s}_{\mathrm{env}}).
\label{eq:prop2_cov_match_inverse_si}
\end{equation}
Multiplying Eq.~\eqref{eq:prop2_mean_match_inverse_si} by $\mathbf{A}_t$ and rearranging Eq.~\eqref{eq:prop2_cov_match_inverse_si} yields
\begin{equation}
\boldsymbol{\mu}_q(\mathbf{s}_{\mathrm{env}},t)
=
\mathbf{A}_t\boldsymbol{\mu}_c(\mathbf{s}_{\mathrm{env}}),
\label{eq:prop2_mean_match_si}
\end{equation}
and
\begin{equation}
\boldsymbol{\Sigma}_q(\mathbf{s}_{\mathrm{env}},t)
=
\mathbf{A}_t\boldsymbol{\Sigma}_c(\mathbf{s}_{\mathrm{env}})\mathbf{A}_t^\top
-
\boldsymbol{\Sigma}_t,
\label{eq:prop2_cov_match_rearranged_si}
\end{equation}
which are exactly the matching conditions in Eq.~\eqref{eq:prop2_matching_si}.

This proves the proposition.
\end{proof}
\setcounter{theorem}{2}
\begin{corollary}[Critical diffusion-time bound]
\label{cor:critical_time_bound_si}
Under the assumptions of Proposition~2, the construction is well-posed if and only if $\boldsymbol{\Sigma}_q(\mathbf{s}_{\mathrm{env}},t)$ is positive-semidefinite 
or equivalently,
\begin{equation}
\boldsymbol{\Sigma}_t \preceq
\mathbf{A}_t \boldsymbol{\Sigma}_c(\mathbf{s}_{\mathrm{env}})\mathbf{A}_t^\top,
\label{eq:corollary_loewner_condition_si}
\end{equation}
where $\preceq$ denotes positive-semidefinite ordering. Hence exact finite-time matching is possible only up to the largest diffusion time $t_{\max}$ for which Eq.~\eqref{eq:corollary_loewner_condition_si} holds.
\end{corollary}

\begin{proof}
By Proposition~2,
\begin{equation}
\boldsymbol{\Sigma}_q(\mathbf{s}_{\mathrm{env}},t)
=
\mathbf{A}_t\boldsymbol{\Sigma}_c(\mathbf{s}_{\mathrm{env}})\mathbf{A}_t^\top
-
\boldsymbol{\Sigma}_t.
\end{equation}
A Gaussian covariance is valid if and only if it is positive semidefinite. Therefore,
\(
\boldsymbol{\Sigma}_q(\mathbf{s}_{\mathrm{env}},t)\succeq 0
\)
is equivalent to
\(
\boldsymbol{\Sigma}_t \preceq
\mathbf{A}_t \boldsymbol{\Sigma}_c(\mathbf{s}_{\mathrm{env}})\mathbf{A}_t^\top
\),
which gives the stated bound.
\end{proof}

\subsection{Replica exchange and multi-state reweighting}
\label{sec:si_re_mbar}

We now derive the replica-exchange acceptance ratio and the reduced potentials used for multi-state reweighting. Both constructions are naturally expressed in terms of the same effective energy associated with the joint target.

Let $\mathbf{z}=(\mathbf{s},\mathbf{x})$ denote the augmented state, where $\mathbf{x}=(\mathbf{x}_1,\dots,\mathbf{x}_K)$. Consider a ladder of replicas indexed by $k$, each associated with diffusion time $t_k$ and unnormalized target
\begin{equation}
\widetilde{\pi}_k(\mathbf{z})
=
q_{\mathrm{ctx}}(\mathbf{s},t_k)
\prod_{i=1}^K
p_i(\mathbf{x}_i)\,
q_{t_k}^{(i)}\!\big(\Phi_i(\mathbf{s})\mid \mathbf{x}_i\big).
\label{eq:si_replica_target}
\end{equation}
The corresponding reduced potential can be written as
\begin{equation}
\widetilde{u}_k(\mathbf{z})
=
-\log q_{\mathrm{ctx}}(\mathbf{s},t_k)
-
\sum_{i=1}^K \log p_i(\mathbf{x}_i)
-
\sum_{i=1}^K \log q_{t_k}^{(i)}\!\big(\Phi_i(\mathbf{s})\mid \mathbf{x}_i\big),
\label{eq:si_reduced_potential_full}
\end{equation}
up to an additive constant that depends only on the replica state $k$.

When proposing a swap between replicas $k$ and $\ell$, with current states $\mathbf{z}^{(k)}$ and $\mathbf{z}^{(\ell)}$, the Metropolis--Hastings \cite{metropolis1953equation} acceptance probability is
\begin{equation}
\alpha_{\mathrm{swap}}
=
\min\!\left\{
1,\;
\exp\!\Big[
-\widetilde{u}_k(\mathbf{z}^{(\ell)})
-\widetilde{u}_\ell(\mathbf{z}^{(k)})
+\widetilde{u}_k(\mathbf{z}^{(k)})
+\widetilde{u}_\ell(\mathbf{z}^{(\ell)})
\Big]
\right\}.
\label{eq:si_re_swap_general}
\end{equation}

Because the prior factors $p_i(\mathbf{x}_i)$ are identical across replicas, they cancel exactly in the swap ratio. The acceptance rule therefore simplifies to
\begin{equation}
\alpha_{\mathrm{swap}}
=
\min\!\left\{
1,\;
\frac{
q_{\mathrm{ctx}}(\mathbf{s}^{(\ell)},t_k)\,
q_{\mathrm{ctx}}(\mathbf{s}^{(k)},t_\ell)\,
\prod_i q_{t_k}^{(i)}\!\big(\Phi_i(\mathbf{s}^{(\ell)})\mid \mathbf{x}_i^{(\ell)}\big)\,
\prod_i q_{t_\ell}^{(i)}\!\big(\Phi_i(\mathbf{s}^{(k)})\mid \mathbf{x}_i^{(k)}\big)
}{
q_{\mathrm{ctx}}(\mathbf{s}^{(k)},t_k)\,
q_{\mathrm{ctx}}(\mathbf{s}^{(\ell)},t_\ell)\,
\prod_i q_{t_k}^{(i)}\!\big(\Phi_i(\mathbf{s}^{(k)})\mid \mathbf{x}_i^{(k)}\big)\,
\prod_i q_{t_\ell}^{(i)}\!\big(\Phi_i(\mathbf{s}^{(\ell)})\mid \mathbf{x}_i^{(\ell)}\big)
}
\right\}.
\label{eq:si_re_swap_simplified}
\end{equation}
This is the practical advantage of replica exchange in GG-PA: swap moves require evaluating only the explicit context factors and the known forward kernels, not the learned prior densities themselves.

\subsection{GG-PA with diffusion-time replica exchange}
\label{si:ggpa_re_algorithm}

For replica exchange, we run GG-PA simultaneously at a ladder of diffusion times
\(0<t_1<\cdots<t_R<1\) (Fig.~\ref{fig:si_re_schematic}). Let
\(\mathbf{z}^{(r)}=(\mathbf{s}^{(r)},\{\mathbf{x}_i^{(r)}\}_{i=1}^K)\)
denote the augmented state of replica \(r\). For compactness, define the
replica-specific exchange weight
\begin{equation}
\rho_r\!\left(\mathbf{s},\{\mathbf{x}_i\}\right)
=
q_{\mathrm{ctx}}^{(r)}(\mathbf{s})
\prod_{i=1}^K
q_{t_r}^{(i)}\!\left(\Phi_i(\mathbf{s})\mid \mathbf{x}_i\right),
\label{eq:si_exchange_weight}
\end{equation}
where \(q_{\mathrm{ctx}}^{(r)}(\mathbf{s})=q_{\mathrm{ctx}}(\mathbf{s},t_r)\)
for a standard context schedule. In fixed-context replica ladders, one may instead
use the same context factor for all replicas.

\begin{figure}[t]
    \centering
    \includegraphics[width=\linewidth]{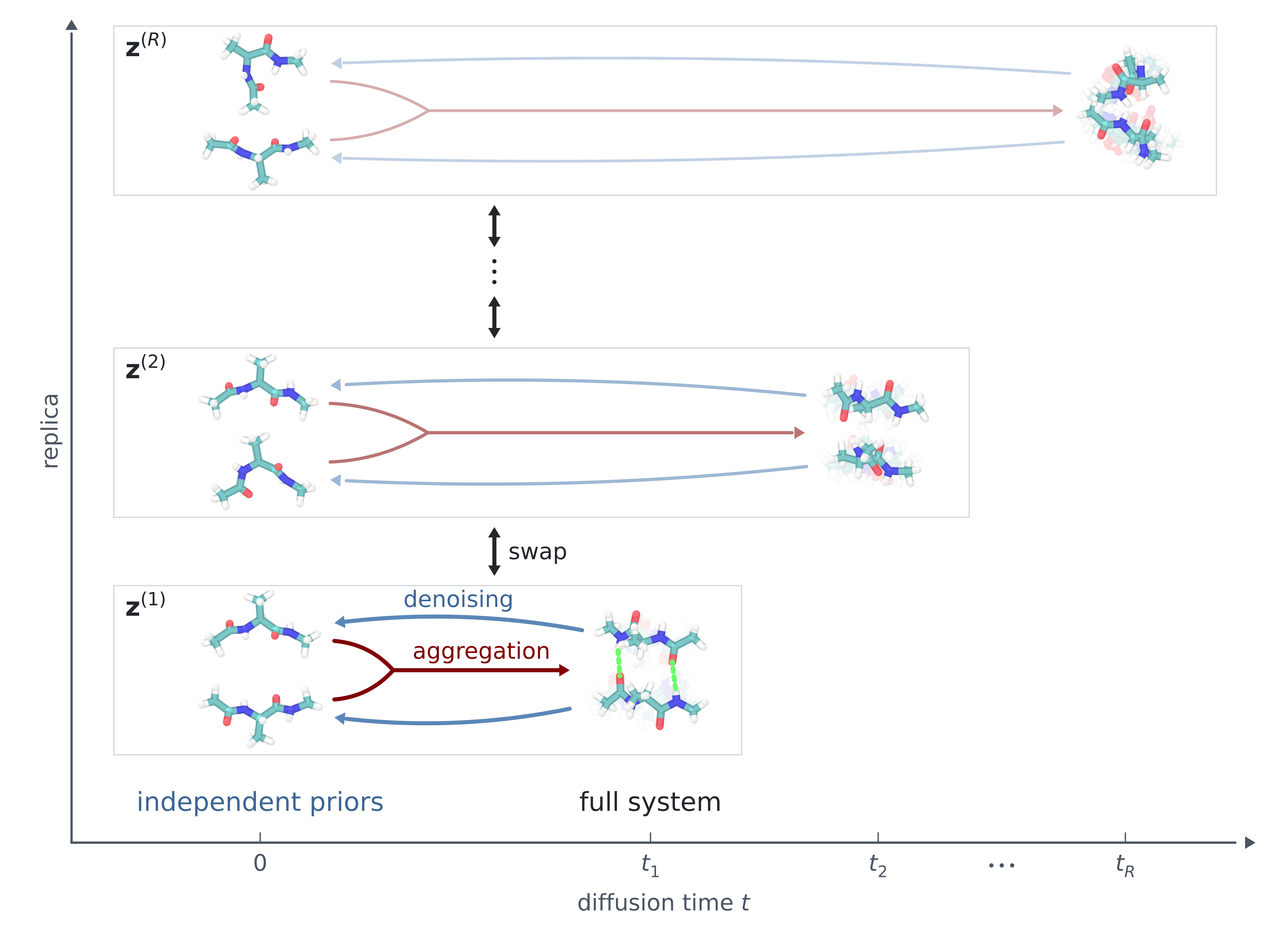}
    \caption{\textbf{GG-PA with diffusion-time replica exchange.} Each box is a replica $\mathbf{z}^{(r)}=(\mathbf{s}^{(r)},\{\mathbf{x}_i^{(r)}\})$ at diffusion time $t_r$, with the prior variables $\mathbf{x}_i^{(r)}$ on the left and the full-system state $\mathbf{s}^{(r)}$ on the right. GG-PA alternates between denoising $\mathbf{s}^{(r)}$ into the $\mathbf{x}_i^{(r)}$ (blue) and aggregating the $\mathbf{x}_i^{(r)}$ into $\mathbf{s}^{(r)}$ (red). Neighboring replicas swap with probability $\alpha_{r,r+1}$ (Eq.~\eqref{eq:si_re_accept_compact}). Results are reported from $\mathbf{s}^{(1)}$, the full-system state of the lowest-noise replica $\mathbf{z}^{(1)}$ (bottom). Shown for the AD dimer.}

    \label{fig:si_re_schematic}
\end{figure}

A proposed swap between neighboring replicas \(r\) and \(r+1\) is accepted with
probability
\begin{equation}
\alpha_{r,r+1}
=
\min\!\left[
1,\,
\frac{
\rho_r\!\left(\mathbf{s}^{(r+1)},\{\mathbf{x}_i^{(r+1)}\}\right)
\rho_{r+1}\!\left(\mathbf{s}^{(r)},\{\mathbf{x}_i^{(r)}\}\right)
}{
\rho_r\!\left(\mathbf{s}^{(r)},\{\mathbf{x}_i^{(r)}\}\right)
\rho_{r+1}\!\left(\mathbf{s}^{(r+1)},\{\mathbf{x}_i^{(r+1)}\}\right)
}
\right].
\label{eq:si_re_accept_compact}
\end{equation}
The intractable prior densities \(p_i(\mathbf{x}_i)\) cancel from this ratio.

\begin{algorithm}[t]
\caption{GG-PA with Diffusion-Time Replica Exchange (GG-PA-RE)}
\label{alg:ggpa_re}
\begin{algorithmic}[1]
\REQUIRE Context factors $\{q_{\mathrm{ctx}}^{(r)}\}_{r=1}^R$, pretrained diffusion priors $\{p_i\}_{i=1}^K$, projectors $\{\Phi_i\}_{i=1}^K$, diffusion-time ladder $0<t_1<\cdots<t_R<1$, sweeps $N$
\STATE Initialize augmented states $\mathbf{z}^{(0,r)}=(\mathbf{s}^{(0,r)},\{\mathbf{x}_i^{(0,r)}\}_{i=1}^K)$ for $r=1,\dots,R$
\FOR{$n = 1,2,\dots,N$}
    \FOR{$r = 1$ \TO $R$ \textbf{in parallel}}
        \FOR{$i = 1$ \TO $K$ \textbf{in parallel}}
            \STATE Sample $\mathbf{x}_i^{(n,r)} \sim p_i\!\left(\mathbf{x}_i \mid \Phi_i(\mathbf{s}^{(n-1,r)}), t_r\right)$
        \ENDFOR
        \STATE Sample $\mathbf{s}^{(n,r)}$ under
        \[
        U_{\mathrm{eff}}^{(r)}(\mathbf{s};\{\mathbf{x}_i^{(n,r)}\})
        =
        -\log q_{\mathrm{ctx}}^{(r)}(\mathbf{s})
        -
        \sum_{i=1}^K
        \log q_{t_r}^{(i)}\!\left(\Phi_i(\mathbf{s})\mid \mathbf{x}_i^{(n,r)}\right)
        \]
    \ENDFOR
    \STATE Set $\mathcal{P}_n=\{1,3,5,\dots\}$ if $n$ is odd, and $\mathcal{P}_n=\{2,4,6,\dots\}$ if $n$ is even
    \FOR{$r \in \mathcal{P}_n$ with $r<R$}
        \STATE Compute $\alpha_{r,r+1}$ using Eq.~\eqref{eq:si_re_accept_compact}
        \STATE With probability $\alpha_{r,r+1}$, swap $\mathbf{z}^{(n,r)}$ and $\mathbf{z}^{(n,r+1)}$
    \ENDFOR
\ENDFOR
\RETURN Samples from the target low-diffusion-time replica, optionally together with all replicas for reweighting
\end{algorithmic}
\end{algorithm}

To improve sample efficiency, we further combine replica exchange with the Multistate Bennett Acceptance Ratio (MBAR) estimator \cite{shirts2008statistically}. MBAR requires the reduced potential of every collected sample evaluated under every replica state. Because MBAR is invariant to additive terms that are shared across all replica states for a given sample, the replica-independent prior contribution
\(
-\sum_i \log p_i(\mathbf{x}_i)
\)
may be omitted. The reduced potential used for reweighting can therefore be written as
\begin{equation}
u_k^{\mathrm{MBAR}}(\mathbf{z})
=
-\log q_{\mathrm{ctx}}(\mathbf{s},t_k)
-
\sum_{i=1}^K
\log q_{t_k}^{(i)}\!\big(\Phi_i(\mathbf{s})\mid \mathbf{x}_i\big),
\label{eq:mbar_reduced_potential_multi}
\end{equation}
where the normalization constants of the forward kernels are included in $\log q_{t_k}^{(i)}$.

In practice, MBAR is effective when neighboring replicas exhibit sufficient phase-space overlap, which is typically reflected in moderate swap acceptance rates. Under this condition, samples from larger-$t$ replicas, which mix more rapidly, can be reweighted to improve the effective sample size of the target small-$t$ ensemble while preserving the correct equilibrium expectations.

Unless otherwise noted, MBAR is not applied in the experiments.

\subsection{Specialization to split Gibbs sampling for linear inverse problems}
\label{sec:si_linear_inverse}

The finite-time exactness construction in Proposition~2 applies directly to the standard noisy linear inverse problem studied in plug-and-play (PnP) and split Gibbs sampling methods \cite{venkatakrishnan2013plug,coeurdoux2024plug,wu2024principled}. We briefly spell out this specialization, as it identifies a covariance correction that removes the finite-noise bias present in existing split Gibbs samplers for this problem class.

Consider a single prior ($K=1$) with $\mathbf{x}\in\mathbb{R}^n$, a linear forward operator $\mathbf{H}\in\mathbb{R}^{m\times n}$, and a noisy observation
\begin{equation}
\mathbf{y} = \mathbf{H}\mathbf{x} + \boldsymbol{\eta},
\qquad
\boldsymbol{\eta}\sim\mathcal{N}(\mathbf{0},\boldsymbol{\Sigma}_\eta).
\label{eq:si_lip_obs}
\end{equation}
The target posterior is $p(\mathbf{x}\mid\mathbf{y})\propto p(\mathbf{x})\,\mathcal{N}(\mathbf{y};\mathbf{H}\mathbf{x},\boldsymbol{\Sigma}_\eta)$. Existing split Gibbs methods introduce an auxiliary variable $\mathbf{z}$ with joint target
\begin{equation}
\pi_\rho(\mathbf{x},\mathbf{z}\mid\mathbf{y})
\propto
p(\mathbf{x})\,
\mathcal{N}(\mathbf{y};\mathbf{H}\mathbf{z},\boldsymbol{\Sigma}_{\mathrm{split}})
\exp\!\left(-\tfrac{1}{2\rho^2}\|\mathbf{x}-\mathbf{z}\|^2\right),
\label{eq:si_lip_joint}
\end{equation}
and alternate between a denoising step that samples $\mathbf{x}\mid\mathbf{z}$ using the diffusion prior and a likelihood-consistency step that samples $\mathbf{z}\mid\mathbf{x},\mathbf{y}$ analytically. This corresponds to GG-PA with $K=1$, $\Phi=\mathrm{id}$, and an isotropic forward kernel $q_t(\mathbf{z}\mid\mathbf{x})=\mathcal{N}(\mathbf{z};\mathbf{x},\rho^2 I)$ (setting $\alpha_t=1$ for simplicity).

Marginalizing $\mathbf{z}$ yields
\begin{equation}
\pi_\rho(\mathbf{x}\mid\mathbf{y})
\propto
p(\mathbf{x})\,
\mathcal{N}\!\left(\mathbf{y};\mathbf{H}\mathbf{x},\,\boldsymbol{\Sigma}_{\mathrm{split}}+\rho^2 \mathbf{H}\mathbf{H}^\top\right).
\label{eq:si_lip_marginal}
\end{equation}
Standard implementations set $\boldsymbol{\Sigma}_{\mathrm{split}}=\boldsymbol{\Sigma}_\eta$, which inflates the effective measurement covariance to $\boldsymbol{\Sigma}_\eta+\rho^2 \mathbf{H}\mathbf{H}^\top$. The $\mathbf{x}$-marginal then differs from the true posterior for any $\rho>0$, and exactness requires $\rho\to 0$.

By direct application of Proposition~2, the choice
\begin{equation}
\boldsymbol{\Sigma}_{\mathrm{split}}
=
\boldsymbol{\Sigma}_\eta - \rho^2 \mathbf{H}\mathbf{H}^\top
\label{eq:si_lip_matched}
\end{equation}
yields $\pi_\rho(\mathbf{x}\mid\mathbf{y})\propto p(\mathbf{x})\,\mathcal{N}(\mathbf{y};\mathbf{H}\mathbf{x},\boldsymbol{\Sigma}_\eta)$ for every admissible $\rho$, so the $\mathbf{x}$-marginal is the exact posterior without requiring $\rho\to 0$.

The covariance is positive definite if and only if $\rho^2 \mathbf{H}\mathbf{H}^\top \prec \boldsymbol{\Sigma}_\eta$, which gives the admissible window
\begin{equation}
\rho < \rho_{\max}
\equiv
\left[
\lambda_{\max}\!\left(
\boldsymbol{\Sigma}_\eta^{-1/2}\mathbf{H}\mathbf{H}^\top\boldsymbol{\Sigma}_\eta^{-1/2}
\right)
\right]^{-1/2}.
\label{eq:si_lip_rhomax}
\end{equation}
In the white-noise case $\boldsymbol{\Sigma}_\eta=\sigma_\eta^2 I$, this simplifies to $\rho_{\max}=\sigma_\eta/\|\mathbf{H}\|_2$.

\paragraph{Mixing within the admissible window.}
Since $\rho$ no longer affects the stationary distribution, it controls only the Gibbs transition kernel and therefore the mixing rate. To make this precise, consider the scalar Gaussian case: $x\sim\mathcal{N}(0,\sigma_x^2)$, $y=hx+\eta$, $\eta\sim\mathcal{N}(0,\sigma_\eta^2)$, with $\rho_{\max}=\sigma_\eta/|h|$. Composing the two Gibbs steps yields an AR(1) chain $x^+ = \varphi_\rho\, x + c_\rho + \xi$, where the lag-one autocorrelation is
\begin{equation}
\varphi_\rho
=
\frac{\sigma_x^2}{\sigma_x^2+\rho^2}
\left(1-\frac{\rho^2}{\rho_{\max}^2}\right).
\label{eq:si_lip_ar1}
\end{equation}
One can verify that $d\varphi_\rho/d(\rho^2)<0$ for all admissible $\rho$, so that larger $\rho$ monotonically decreases the autocorrelation: $\varphi_\rho\to 1$ as $\rho\to 0$ (sticky chain) and $\varphi_\rho\to 0$ as $\rho\to\rho_{\max}$ (near-independent sampler, but numerically singular). This parallels the role of the diffusion time $t$ in the quadratic systems of the main text: within the admissible window, increasing the split noise improves exploration without sacrificing exactness, while values too close to the noise ceiling become numerically unstable. We verify these predictions on a two-dimensional Gaussian mixture inverse problem in Appendix~\ref{sec:si_gmm_verification}.

\section{Experimental Details}
\label{sec:experimental_details}

This appendix summarizes the experimental instantiations of GG-PA for the three systems studied in the main text. 
We first describe the conventions shared across systems, and then provide system-specific details.
\subsection{Common implementation and training conventions}
\label{sec:common_model_training}

For the coupled double-well and $\phi^4$ experiments, the pretrained priors share the same backbone architecture and diffusion parameterization. In both cases, the prior is a residual MLP with sinusoidal time embeddings, trained under a variance-preserving (VP) forward process with a cosine noise schedule and the standard $v$-parameterization objective. Unless otherwise stated, the backbone consists of $3$ residual blocks with hidden width $32$ and time-embedding dimension $16$, and is optimized with Adam. Both priors model scalar variables; the system-specific sections below specify the remaining hyperparameters, and sampling-time discretization choices.

For these VP priors, the GG-PA prior update is implemented by initializing the reverse diffusion process at the current noisy observation and integrating from the chosen diffusion time to $0$. The Direct Diffusion baseline uses the same reverse process, but is started from pure noise rather than from a system-dependent noisy observation.

Across all GG-PA experiments, one outer sweep consists of a prior update followed by an aggregation update that enforces consistency with the explicit context. In the coupled double-well and $\phi^4$ experiments, the aggregation update is exact because the forward kernel and the explicit context together define Gaussian conditional distributions over the noisy variables. In replica-exchange runs, neighboring swaps are attempted after each outer sweep using the Metropolis--Hastings rule derived in Appendix~\ref{sec:si_re_mbar}. Unless otherwise stated, reported observables are computed after discarding an initial burn-in period from each chain.

For both the coupled double-well and $\phi^4$ experiments, we work in dimensionless units with inverse temperature fixed to $\beta=1$. Accordingly, $\beta$ is absorbed into the system parameters below and is not written explicitly in the system-specific formulas.

\subsubsection{Statistical uncertainty and error bars}
We use experiment-specific uncertainty conventions depending on the observable and sampling protocol. For the coupled double-well system, error bars for the JS divergence and the IAT summarize variability across independent trajectories/chains at fixed diffusion time after burn-in. For the 2D $\phi^4$ results, error bars denote autocorrelation-corrected standard errors computed from the production time series using the integrated autocorrelation time of the corresponding observable. For the alanine dipeptide systems, error bars reflect the variability across independent chains.

\subsubsection{Compute resources and software}
All experiments were performed on a single workstation with a 16-core CPU and an NVIDIA GeForce RTX 5080 GPU (16 GB). The alanine dipeptide simulations ran on the CPU with OpenMM \cite{eastman2023openmm}, and the other experiments ran on the GPU.

Training each diffusion prior takes less than 10 minutes. Production runs take a few minutes for the coupled double-well and $\phi^4$ systems and up to about 9 h for the AD dimer.

\subsection{Coupled double-well system}
\label{sec:double_well_details}

\subsubsection{System and target distribution}

The coupled double-well system is defined by the total potential
\begin{equation}
U_{\mathrm{tot}}(x_{\mathrm{sys}},x_{\mathrm{env}})
=
\underbrace{A(x_{\mathrm{sys}}^2-1)^2}_{U_{\mathrm{sys}}(x_{\mathrm{sys}})}
+
\underbrace{\frac{k_c}{2}(x_{\mathrm{sys}}-x_{\mathrm{env}})^2}_{U_{\mathrm{int}}(x_{\mathrm{sys}},x_{\mathrm{env}})}
+
\underbrace{\frac{k_b}{2}(x_{\mathrm{env}}-u_{\mathrm{eq}})^2}_{U_{\mathrm{env}}(x_{\mathrm{env}})}.
\end{equation}
We use $A=8$, $k_c=4$, $k_b=1$, and $u_{\mathrm{eq}}=1$.

As a physical baseline, we simulate overdamped Langevin dynamics using the Euler--Maruyama update
\begin{equation}
\mathbf{x}_{n+1}
=
\mathbf{x}_n
-
\nabla U_{\mathrm{tot}}(\mathbf{x}_n)\,\Delta t
+
\sqrt{2\Delta t}\,\boldsymbol{\eta}_n,
\qquad
\boldsymbol{\eta}_n\sim\mathcal{N}(\mathbf{0},I),
\end{equation}
with $\mathbf{x}_n=(x_{\mathrm{sys},n},x_{\mathrm{env},n})$. We run $5$ independent trajectories initialized at $(x_{\mathrm{sys}},x_{\mathrm{env}})=(1.0,1.0)$ with integration timestep $\Delta t=5\times 10^{-3}$. Each trajectory records $25{,}000$ samples with a recording interval of $1.0$ in simulation time (every $200$ integration steps), and the first $20\%$ of samples are discarded as burn-in.

To quantify mixing, we compute the integrated autocorrelation time (IAT) of the basin indicator
\begin{equation}
A_n=\mathrm{sgn}(x_{\mathrm{sys},n}).
\end{equation}
The IAT is estimated using the statistical inefficiency implementation in \texttt{pymbar} \cite{shirts2008statistically}, and reported values are averaged over independent chains.

\subsubsection{GG-PA and replica-exchange setup}

The diffusion prior acts only on the system coordinate $x_{\mathrm{sys}}$, while $x_{\mathrm{env}}$ remains part of the explicit physical context. Because the interaction term is quadratic, the finite-$t$ context schedule can be constructed in closed form by Gaussian deconvolution of the physical coupling against the VP forward kernel
\begin{equation}
q_t(s\mid x_{\mathrm{sys}})
=
\mathcal{N}\!\left(s;\alpha(t)x_{\mathrm{sys}},\sigma^2(t)\right).
\end{equation}
The resulting exact finite-$t$ context schedule may be written as
\begin{equation}
U_{\mathrm{ctx}}(s,x_{\mathrm{env}};t)
=
\frac{1}{2}\kappa_t\big(s-\alpha(t)x_{\mathrm{env}}\big)^2
+
\frac{k_b}{2}(x_{\mathrm{env}}-u_{\mathrm{eq}})^2,
\end{equation}
with effective precision
\begin{equation}
\kappa_t=
\frac{k_c}{\alpha^2(t)-k_c\sigma^2(t)}.
\end{equation}
This construction is valid only when
\begin{equation}
\alpha^2(t)-k_c\sigma^2(t)>0.
\end{equation}
Under the VP parameterization $\sigma^2(t)=1-\alpha^2(t)$, this becomes
\begin{equation}
\alpha^2(t)>\frac{k_c}{1+k_c}.
\end{equation}
For $k_c=4$, this yields $\alpha^2(t)>0.8$, corresponding to a maximal valid diffusion time of approximately $t_{\max}\approx 0.28$ under the cosine schedule.

In fixed-$t$ GG-PA runs with $t\le t_{\max}$, including the production settings $t\in\{0.1,0.2\}$, we use the exact finite-$t$ context schedule above. For larger diffusion times in fixed-$t$ runs, where the deconvolved Gaussian becomes ill-posed, we revert to the unannealed physical context
\begin{equation}
U_{\mathrm{ctx}}(s,x_{\mathrm{env}})
=
\frac{k_c}{2}(s-x_{\mathrm{env}})^2
+
\frac{k_b}{2}(x_{\mathrm{env}}-u_{\mathrm{eq}})^2.
\end{equation}

Given the denoised clean variable from the prior update, the aggregation step samples the noisy system variable and the environment coordinate jointly from the exact conditional distribution induced by the effective potential above. Because both the forward kernel and the context are quadratic, this conditional is Gaussian and the aggregation step is implemented exactly. Production runs use $250$ parallel chains for $500$ outer sweeps.

To improve mixing of the small-$t$ target ensemble, we use replica exchange over diffusion time with the ladder
\begin{equation}
t\in\{0.1,\,0.2,\,0.4,\,0.8\}.
\end{equation}
In these replica-exchange runs, all replicas share the exact finite-$t$ context schedule corresponding to the production time $t_{\mathrm{prod}}=0.1$, and differ only in the forward-process noise level used in the prior update.

\subsubsection{Prior model}

The prior is trained on the isolated one-dimensional system corresponding to $k_c=0$,
\begin{equation}
U_{\mathrm{sys}}(x_{\mathrm{sys}})=8(x_{\mathrm{sys}}^2-1)^2.
\end{equation}
The score model uses the shared residual MLP architecture described in Appendix~\ref{sec:common_model_training}. We use a VP forward process with a cosine noise schedule over $T=1000$ discrete timesteps, bounded by $\beta_{\min}=10^{-4}$ and $\beta_{\max}=0.02$, and train the model using the $v$-parameterization.

\subsubsection{Training details}

The training set consists of $100{,}000$ samples generated by direct rejection sampling from $\exp[-U_{\mathrm{sys}}(x_{\mathrm{sys}})]$. To enforce exact reflection symmetry, we augment the dataset by applying $x_{\mathrm{sys}}\mapsto -x_{\mathrm{sys}}$, yielding $200{,}000$ training samples in total.

Training uses an SNR-weighted mean-squared loss with weight factor $5.0$, Adam optimizer, batch size $256$, initial learning rate $10^{-3}$, and cosine annealing with $\eta_{\min}=10^{-5}$ for $100$ epochs.

\subsection{$\phi^4$ lattice model}
\label{sec:phi4_details}

\subsubsection{System and target distribution}

In the main text, the 2D $\phi^4$ model is defined by the Hamiltonian
\begin{equation}
H(\phi)
=
\sum_i(\phi_i^2-1)^2
+
J\sum_{\langle i,j\rangle}(\phi_i-\phi_j)^2
-
h\sum_i\phi_i,
\label{eq:si_phi4_hamiltonian}
\end{equation}
on a periodic $L\times L$ square lattice. Expanding the interaction term gives
\begin{equation}
H(\phi)
=
\sum_i \left[\phi_i^4 + (4J-2)\phi_i^2\right]
-
2J\sum_{\langle i,j\rangle}\phi_i\phi_j
-
h\sum_i\phi_i,
\end{equation}
which is equivalent to the standard lattice $\phi^4$ form
\begin{equation}
H_{\mathrm{std}}(\phi)
=
-K\sum_{\langle i,j\rangle}\phi_i\phi_j
+
\sum_i\left(r\phi_i^2+u\phi_i^4-h\phi_i\right),
\end{equation}
with
\begin{equation}
u=1,\qquad K=2J,\qquad r=4J-2.
\end{equation}

All simulations are performed on a periodic $32\times 32$ lattice. For the zero-field study, we set $h=0$ and scan $J$ across the phase-transition region. For the finite-field scaling analysis, we use
\begin{equation}
h\in\{0.008,\,0.012,\,0.015,\,0.02,\,0.025\},
\end{equation}
and scan $J$ in a scaling window centered at $J_c=0.436$ using the scaling variable
\begin{equation}
x=\frac{J-J_c}{h^{8/15}}.
\end{equation}
For each field value, we use a uniform grid of $10$ points in $x\in[-2,2]$, clipped to the range $J\in[0.10,0.80]$.

As a reference, we use vectorized checkerboard Metropolis Monte Carlo \cite{metropolis1953equation} with proposals
\begin{equation}
\delta \sim \mathrm{Uniform}(-0.5,\,0.5),
\end{equation}
applied on alternating checkerboard masks. For the main zero-field scan, we use $50{,}000$ equilibration sweeps and $150{,}000$ measurement sweeps per point.

Unless otherwise stated, each GG-PA production run uses $10{,}000$ recorded sweeps, magnetization is recorded every sweep, and the first $30\%$ of samples are discarded as burn-in. In the critical region, selected scan points are extended with longer runs to ensure enough effective samples.

\subsubsection{GG-PA and replica-exchange setup}

Here the clean variable is the denoised field $\phi$, and the noisy variable is the auxiliary field $\psi$. Because the inter-site coupling is quadratic, the aggregation step again reduces to exact Gaussian sampling.

We choose the finite-$t$ context schedule factor
\begin{equation}
q_{\mathrm{ctx}}(\psi;t)
\propto
\exp\!\left[
-\frac{1}{2}\psi^\top Q(t)\psi
+
\frac{h}{\alpha(t)}\sum_i \psi_i
\right].
\label{eq:si_phi4_ctx}
\end{equation}
Under periodic boundary conditions, $Q(t)$ is diagonal in Fourier space. If $\lambda_k$ denotes the eigenvalue of the discrete lattice Laplacian at mode $k$, then for nonzero modes
\begin{equation}
q_k(t)
=
\frac{2J\lambda_k}{\alpha^2(t)-2J\lambda_k\,\sigma^2(t)},
\qquad
q_0(t)=0.
\end{equation}
Therefore, exact finite-time matching requires
\begin{equation}
\alpha^2(t)-2J\lambda_k\,\sigma^2(t)\ge 0.
\end{equation}
Since $\lambda_k\in[0,8]$ on the 2D periodic square lattice, the most restrictive condition is
\begin{equation}
\alpha^2(t)-16J\,\sigma^2(t)\ge 0,
\label{eq:si_phi4_validity}
\end{equation}
or equivalently,
\begin{equation}
\alpha^2(t)\ge \frac{16J}{1+16J}
\end{equation}
under the VP parameterization $\sigma^2(t)=1-\alpha^2(t)$.

Given a denoised clean field $\phi$, the aggregation step samples
\begin{equation}
p(\psi\mid \phi,t)\propto q_t(\psi\mid \phi)\,q_{\mathrm{ctx}}(\psi;t),
\end{equation}
which is Gaussian in $\psi$. Because both factors are diagonal in Fourier space, this conditional can be sampled exactly mode by mode. For each Fourier mode,
\begin{equation}
A_k=\sigma^{-2}(t)+q_k(t),
\qquad
v_k=A_k^{-1},
\end{equation}
and the conditional mean is
\begin{equation}
\hat{\mu}_k
=
v_k\left(\frac{\alpha(t)}{\sigma^2(t)}\hat{\phi}_k+\hat{b}_k\right),
\end{equation}
where $\hat{b}_k=0$ for $k\neq 0$ and $\hat{b}_0=hL/\alpha(t)$ under the orthonormal FFT convention. We then sample
\begin{equation}
\hat{\psi}_k=\hat{\mu}_k+\sqrt{v_k}\,\hat{z}_k
\end{equation}
and transform back to real space by inverse FFT. This gives an exact aggregation step for the quadratic context.

To mitigate critical slowing down, we use replica exchange over diffusion time. In these runs, the context precision $Q$ is fixed at the production diffusion time $t_{\mathrm{prod}}=0.1$, while the forward-process noise level varies across replicas. We use $R=48$ replicas with geometrically spaced diffusion times between
\begin{equation}
t_{\mathrm{prod}}=0.1
\qquad\text{and}\qquad
t_{\max}=0.6,
\end{equation}
which yields stable mixing and typical swap acceptance rates of approximately $30\%$.

\subsubsection{Prior model}

The prior is trained on the one-dimensional on-site distribution
\begin{equation}
p(\phi)\propto \exp[-(\phi^2-1)^2].
\end{equation}
The score model uses the same residual MLP backbone and sampling protocol as the double-well prior, with the shared architecture described in Appendix~\ref{sec:common_model_training}.

\subsubsection{Training details}

Training data are generated by direct rejection sampling from the on-site target distribution. The model is trained with an SNR-weighted mean-squared loss using Adam with learning rate $10^{-3}$.

Because the lattice experiments operate primarily in the small-noise regime and the system is in a near critical condition, we notice that the quality of the checkpoint is crucial for the results. We choose the checkpoint by comparing the variance of the posterior denoised samples against the exact one-dimensional posterior
\begin{equation}
p(\phi\mid \psi)
\propto
\exp[-(\phi^2-1)^2]\,
\mathcal{N}(\psi;\alpha(t)\phi,\sigma^2(t)),
\end{equation}
and select the model whose variance ratio is closest to $1$ in this regime. In later tests, we found that target score matching \cite{de2024target} trains more stably at small noise than the denoising loss used here. Target score matching uses the analytical score $\nabla\log p(\phi)$, which is available as a training target for the systems we study.
\subsection{Alanine dipeptide systems}
\label{sec:ad_details}

\subsubsection{Shared simulation and sampling setup}

All molecular dynamics simulations for the alanine dipeptide (AD) systems are performed in OpenMM \cite{eastman2023openmm} on CPU. Unless otherwise noted, we use the force-field files \texttt{amber99sbildn.xml} \cite{hornak2006comparison,lindorff2010improved} and \texttt{tip3p.xml} \cite{jorgensen1983comparison}, the NVT ensemble, a \texttt{LangevinMiddleIntegrator} at $300$ K with friction coefficient $1.0~\mathrm{ps}^{-1}$, timestep $2~\mathrm{fs}$, \texttt{NoCutoff} nonbonded interactions, and \texttt{HBonds} constraints.

In the AD systems, the full-system state remains a full atomistic MD configuration, while the diffusion prior acts only on the projected backbone torsions of each alanine dipeptide. For a system containing $K$ alanine dipeptides, we extract
\begin{equation}
\Phi_i(\mathbf{s})=(\phi_i,\psi_i),
\qquad
i=1,\dots,K.
\end{equation}
At the beginning of each GG-PA sweep, the current projected torsions are used as noisy inputs to the pretrained torsion diffusion model at diffusion time $t$, and reverse diffusion is run independently for each molecule to obtain denoised torsion anchors
\begin{equation}
\mathbf{x}_i=(\phi_i^\ast,\psi_i^\ast),
\qquad
i=1,\dots,K.
\end{equation}
For the dimer system, the same monomer prior is applied independently to the two monomers.

In all alanine experiments, we do not anneal the explicit physical context, i.e.
\begin{equation}
q_{\mathrm{ctx}}(\mathbf{s},t)=\pi_{\mathrm{ctx}}(\mathbf{s}),
\end{equation}
so the aggregation target is
\begin{equation}
\pi_t(\mathbf{s}\mid \{\mathbf{x}_i\})
\propto
\pi_{\mathrm{ctx}}(\mathbf{s})
\prod_{i=1}^K q_t^{(i)}\!\big(\Phi_i(\mathbf{s})\mid \mathbf{x}_i\big),
\end{equation}
with corresponding effective potential
\begin{equation}
U_{\mathrm{eff}}(\mathbf{s};\{\mathbf{x}_i\},t)
=
-\log \pi_{\mathrm{ctx}}(\mathbf{s})
-
\sum_{i=1}^K \log q_t^{(i)}\!\big(\Phi_i(\mathbf{s})\mid \mathbf{x}_i\big).
\end{equation}
Operationally, the aggregation step is implemented in OpenMM by adding wrapped torsional restraints centered at the denoised anchors and running restrained MD under the effective potential above. Because the reverse sampler stops at $\sigma_{\min}$, the restraint kernel uses the variance $\sigma_t^2-\sigma_{\min}^2$.

Unless otherwise noted, all alanine runs use production diffusion time $t=0.1$ and $100$ restrained-MD steps per outer sweep, and the first $10\%$ of each chain is discarded during analysis. Each recorded sweep therefore corresponds to
\begin{equation}
100\times 2~\mathrm{fs}=200~\mathrm{fs}.
\end{equation}
To reduce double counting between the torsion prior and the explicit atomistic context, the runs remove the internal dihedral and nonbonded interactions associated with the atoms entering the modeled backbone dihedrals.

\subsubsection{Prior model}

For the alanine dipeptide systems, the diffusion prior is trained only on the backbone torsion angles $(\phi,\psi)$, using $32{,}000$ frames from $300$~K vacuum MD of the monomer. Because these variables live on a torus, we use a wrapped-Gaussian forward process \cite{de2022riemannian, jing2022torsional,huang2022riemannian}. For a single alanine dipeptide, the clean torsion vector is $\mathbf{x}_0\in[-\pi,\pi)^d$ with $d=2$, corresponding to the pair $(\phi,\psi)$. The forward process is
\begin{equation}
\mathbf{x}_t
=
\mathrm{wrap}\!\left(\mathbf{x}_0+\sigma_t\boldsymbol{\epsilon}\right),
\qquad
\boldsymbol{\epsilon}\sim\mathcal{N}(\mathbf{0},I_d),
\end{equation}
where $\mathrm{wrap}(\cdot)$ maps each angle to $[-\pi,\pi)$ and
\begin{equation}
\sigma_t
=
\sigma_{\min}\left(\frac{\sigma_{\max}}{\sigma_{\min}}\right)^t,
\qquad
\sigma_{\min}=0.1,\quad \sigma_{\max}=3.0.
\end{equation}
The corresponding conditional density is the wrapped Gaussian
\begin{equation}
q_t(\mathbf{x}_t\mid \mathbf{x}_0)
=
\sum_{\mathbf{m}\in\mathbb{Z}^d}
\mathcal{N}\!\left(\mathbf{x}_t;\mathbf{x}_0+2\pi\mathbf{m},\sigma_t^2 I_d\right).
\end{equation}

The same wrapped-Gaussian kernel is used in the aggregation step as a torsional restraint, so that the diffusion prior and the MD-based aggregation update are defined with respect to a consistent corruption model.

The score network uses the same residual-MLP design principle as above, but adapted to periodic angular inputs. Each torsion pair is represented by its sine and cosine,
\begin{equation}
(\phi,\psi)\mapsto (\sin\phi,\cos\phi,\sin\psi,\cos\psi),
\end{equation}
which removes the discontinuity at $\pm\pi$. We use hidden width $64$, $3$ hidden layers, dropout rate $0.1$, and sinusoidal time embeddings.

Training uses denoising score matching with diffusion times sampled uniformly from $[0,1]$. For each clean sample $\mathbf{x}_0$, we generate $\mathbf{x}_t$ using the wrapped forward process above and compute the exact wrapped-Gaussian score
\begin{equation}
\mathbf{s}^\star(\mathbf{x}_t,\mathbf{x}_0,t)
=
\nabla_{\mathbf{x}_t}\log q_t(\mathbf{x}_t\mid \mathbf{x}_0).
\end{equation}

Rather than predicting the score directly, the network predicts the preconditioned target $\sigma_t \mathbf{s}^\star$, which improves stability across noise scales. Denoting the network output by $\mathbf{f}_\theta(\mathbf{x}_t,t)$, the training objective is
\begin{equation}
\mathcal{L}(\theta)
=
\mathbb{E}_{\mathbf{x}_0,t,\mathbf{x}_t}
\left[
\left\|
\mathbf{f}_\theta(\mathbf{x}_t,t)
-
\sigma_t \mathbf{s}^\star(\mathbf{x}_t,\mathbf{x}_0,t)
\right\|_2^2
\right].
\end{equation}
We train the model using AdamW with learning rate $10^{-3}$, weight decay $10^{-5}$, batch size $256$, gradient clipping at $1.0$, and cosine annealing for $100$ epochs. The reverse process is discretized using $T=200$ steps in the alanine experiments.

\subsubsection{AD--Na$^+$ system}
\label{subsec:ad_sodium_system}

To prevent complete ion escape in vacuum, we apply a weak harmonic restraint between the AD C$_\alpha$ atom and the Na$^+$ ion. The restraint has equilibrium distance $0.4$ nm and force constant $50.0~\mathrm{kJ\,mol^{-1}\,nm^{-2}}$.

The reference is MD with the full force field and the same restraint, consisting of four trajectories of about $460$~ns. Each prior is sampled with six GG-PA-RE chains, three of them started from a $\phi>0$ configuration. For the data-efficiency study in Fig.~\ref{fig:si_adsodium}, additional priors are trained on $32{,}000$ frames of which a fraction $p\in\{0,1,5,20,50,80,100\}\%$ comes from ion-coupled MD and the rest from vacuum MD. For Direct Diffusion, the O--O distance is evaluated by setting each sampled $(\phi,\psi)$ on the minimized peptide. The main observables are the noisy physical Ramachandran angles $(\phi,\psi)$ and the carbonyl O--O distance. The O--O distance is defined as the distance between the two carbonyl oxygen atoms from residues \texttt{ACE} and \texttt{ALA}. The Ramachandran maps and O--O distribution curves reported for this system are computed from the recorded physical trajectories rather than from the denoised torsion anchors.

\subsubsection{AD dimer system}
\label{subsec:ad_dimer_system}

To prevent trivial dissociation, the dimer runs use an always-on flat-bottom center-of-mass restraint with force constant $50.0~\mathrm{kJ\,mol^{-1}\,nm^{-2}}$ and flat region up to $d_0=0.9$ nm.

In the main experiments, we use the replica-exchange version with a $6$-replica ladder
\begin{equation}
t\in\{0.1,\,0.15,\,0.25,\,0.4,\,0.6,\,0.8\},
\end{equation}
where the production replica is at $t=0.1$. Neighboring swaps are attempted after every outer sweep. We run $28$ chains, half started from a parallel and half from an anti-parallel hydrogen-bonded dimer. The MD reference consists of four $1.06~\mu\mathrm{s}$ trajectories with the same restraint, two from each starting structure.

To describe the symmetry-broken dimer organization, we assign each frame to one of the nine coarse-grained states
\begin{equation}
\{A_{LL},A_{LR},A_{RL},A_{RR},P_{LL},P_{LR},P_{RL},P_{RR},U\},
\end{equation}
where $A$ and $P$ denote anti-parallel and parallel topologies. The monomer basin labels are defined directly in Ramachandran space by two reference-centered basins,
\begin{equation}
\mathbf{c}_L=(-2.4234,\,2.6856),\quad r_L=0.5938,
\qquad
\mathbf{c}_R=(-1.4261,\,1.1006),\quad r_R=0.7878.
\end{equation}
A monomer is assigned to \(L\) or \(R\) if its \((\phi,\psi)\) coordinates lie within the corresponding basin and are closer to that basin center than to the other one; otherwise the monomer is treated as off-basin for the dimer-state assignment. A dimer frame is assigned to \(U\) if it fails the contact, topology, or monomer-basin criteria.

For each dimer frame, we first define a backbone-orientation vector for each monomer. Within the atom list of monomer $m$, we identify the first atom named \texttt{C} and the last atom named \texttt{N}, and set
\begin{equation}
\mathbf{v}_m=\mathbf{r}_{N,\mathrm{last}}^{(m)}-\mathbf{r}_{C,\mathrm{first}}^{(m)},
\qquad m=1,2.
\end{equation}
The relative backbone orientation is quantified by the cosine score
\begin{equation}
\cos\theta
=
\frac{\mathbf{v}_1\cdot \mathbf{v}_2}
{\|\mathbf{v}_1\|\,\|\mathbf{v}_2\|+10^{-12}},
\end{equation}
so that $\cos\theta\approx +1$ corresponds to parallel alignment and $\cos\theta\approx -1$ to antiparallel alignment.

Intermolecular hydrogen bonds are evaluated from the atomistic trajectory using \texttt{mdtraj.baker\_hubbard} \cite{mcgibbon2015mdtraj,baker1984hydrogen} with \texttt{periodic=False}. We retain only intermolecular donor--hydrogen--acceptor triplets, i.e. triplets for which the donor and acceptor atoms belong to different monomers. A hydrogen bond is counted as present when
\begin{equation}
r_{d,a}<0.35~\mathrm{nm},
\qquad
\angle(d,h,a)>120^\circ,
\end{equation}
where $r_{d,a}$ is the donor--acceptor distance. The framewise hydrogen-bond count $N_{\mathrm{HB}}$ is the total number of active intermolecular donor--hydrogen--acceptor triplets in that frame.

In addition, we define a reciprocal-contact condition requiring that both monomers simultaneously donate at least one intermolecular hydrogen bond to the other. Equivalently, a frame satisfies the reciprocal-contact condition if it contains at least two active intermolecular hydrogen bonds and the active donor atoms include at least one donor from each monomer.

A frame is treated as structurally assigned only if
\begin{equation}
|\cos\theta| \ge 0.8,
\qquad
N_{\mathrm{HB}} \ge 2,
\end{equation}
and the reciprocal-contact condition is satisfied. Among such frames, we assign the topology label by
\begin{equation}
A \text{ if } \cos\theta \le -0.8,
\qquad
P \text{ if } \cos\theta \ge 0.8,
\end{equation}
and assign all remaining configurations to $U$.

In the occupancy matrices shown in the main text, occupancies are normalized within each topology, and $LR$ and $RL$ are merged because the two copies are identical. For independent monomers, the occupancies are $p_L^2$, $2p_Lp_R$, and $p_R^2$ in both topologies, with $p_L=0.45$ the fraction of $L$ among $L$ and $R$ frames in the vacuum MD used to train the prior. The unassigned fraction is reported separately. Unless otherwise stated, the dimer analyses are performed on the noisy physical trajectories rather than on the clean denoised torsion anchors.

\section{Additional Results}
\subsection{Verification of covariance-matched split Gibbs on a Gaussian mixture inverse problem}
\label{sec:si_gmm_verification}

To verify the covariance-matched construction derived in Appendix~\ref{sec:si_linear_inverse}, we consider a controlled two-dimensional linear inverse problem with a Gaussian mixture prior. We compare two split Gibbs variants: the \emph{matched} sampler, which uses the corrected split likelihood covariance $\boldsymbol{\Sigma}_{\mathrm{split}}=\boldsymbol{\Sigma}_\eta - \rho^2\mathbf{H}\mathbf{H}^\top$ (Eq.~\eqref{eq:si_lip_matched}), and the \emph{unmatched} sampler, which uses the standard uncorrected choice $\boldsymbol{\Sigma}_{\mathrm{split}}=\boldsymbol{\Sigma}_\eta$.

\paragraph{Setup.}
The prior is a 25-component isotropic Gaussian mixture on $\mathbb{R}^2$ with means on a $5\times 5$ grid $\{-5,-2.5,0,2.5,5\}^2$ and component standard deviation $\sigma_x=0.8$. The forward operator is $\mathbf{H}=[1,\;0]$, so the observation $y=\mathbf{H}x+\eta$ constrains only the first coordinate $x_1$. We set $y=0$ and $\sigma_\eta=2.5$, giving $\rho_{\max}=\sigma_\eta/\|\mathbf{H}\|_2=2.5$. We parameterize the split noise by the normalized level $r=\rho/\rho_{\max}\in(0,1)$. For each method and each value of $r$, we run multiple independent Gibbs chains and pool the post-burn-in samples.

\paragraph{Results.}
Figure~\ref{fig:si_gmm} summarizes the comparison. The top row displays the exact analytic posterior (a) alongside Monte Carlo density estimates from the matched (b) and unmatched (c) samplers at $r=0.95$. Because the observation constrains only $x_1$, the exact posterior concentrates on the central columns of the mixture grid. The matched sampler reproduces this column structure, whereas the unmatched sampler spreads more probability mass into the side columns at $x_1=\pm 5$, reflecting the covariance inflation $\Sigma_{\mathrm{eff}}=\Sigma_\eta+\rho^2\mathbf{H}\mathbf{H}^\top$ identified in Eq.~\eqref{eq:si_lip_marginal}. The one-dimensional marginal of $\mathbf{H}x=x_1$ (d) confirms this: the matched sampler closely tracks the exact posterior density, while the unmatched sampler produces a broader distribution with elevated tails.

Across the admissible window (e), the matched sampler maintains low JS divergence that decreases with $r$, consistent with an exact target and improving chain mixing at larger split noise. At the smallest $r$, both methods show comparable JS values; this is attributable to finite-chain mixing error rather than target bias, as corroborated by the high integrated autocorrelation times at small $r$ (f). In contrast, the unmatched sampler shows monotonically increasing JS divergence as $r$ grows, reflecting the growing target bias from covariance inflation. The IAT of $\mathbf{H}x$ (f) decreases with $r$ for both methods, confirming that larger split noise accelerates mixing. Together, these results demonstrate the practical benefit of covariance matching: one can choose $r$ near $\rho_{\max}$ to improve sampling efficiency while retaining an exact posterior target, whereas the unmatched sampler faces an inherent accuracy--efficiency trade-off.

\begin{figure}[hbt]
    \centering
    \includegraphics[width=\linewidth]{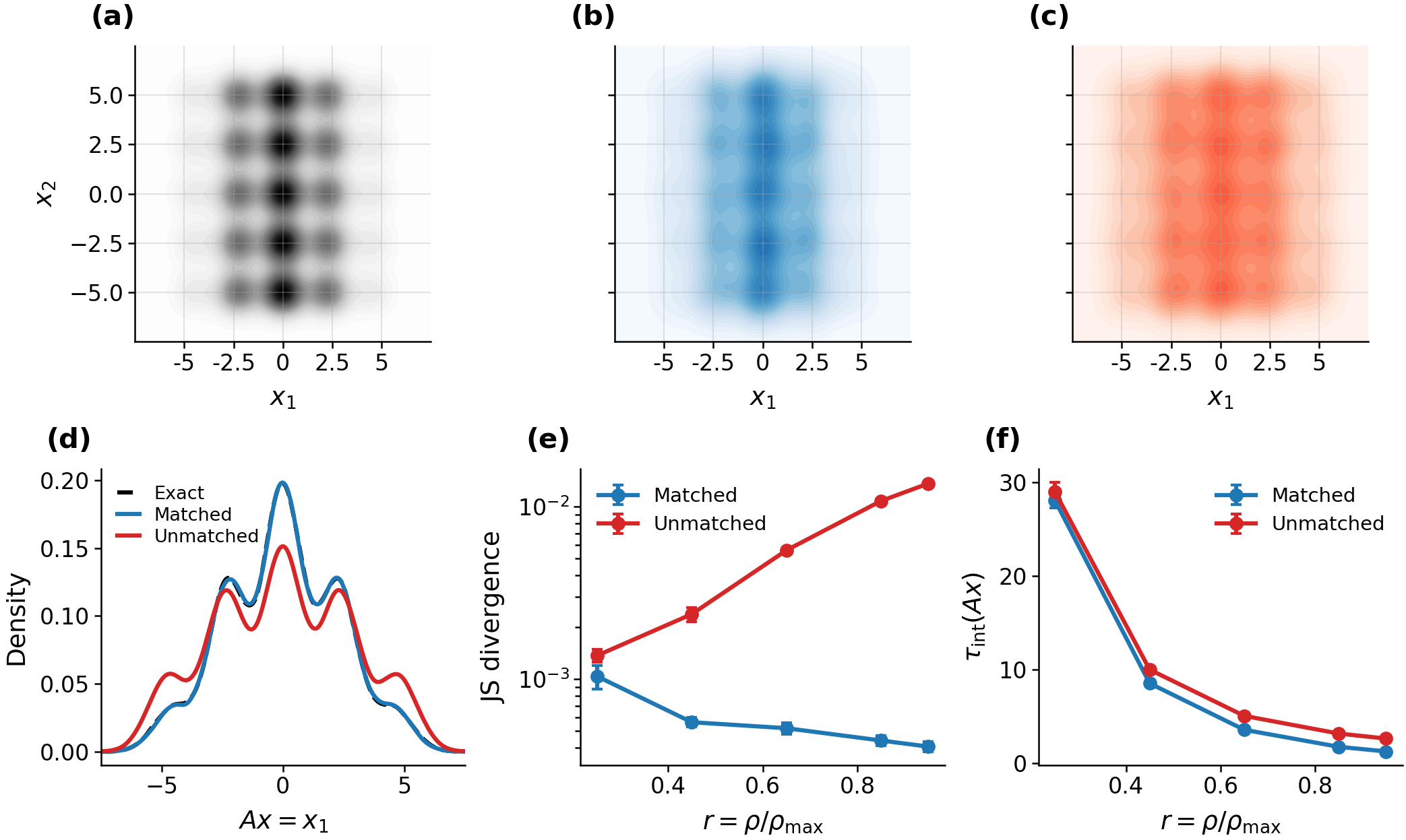}
    \caption{
    \textbf{Covariance-matched vs.\ unmatched split Gibbs sampling on a 2D Gaussian mixture inverse problem.}
    The matched sampler uses the corrected covariance $\boldsymbol{\Sigma}_{\mathrm{split}}=\boldsymbol{\Sigma}_\eta - \rho^2\mathbf{H}\mathbf{H}^\top$; the unmatched sampler uses $\boldsymbol{\Sigma}_{\mathrm{split}}=\boldsymbol{\Sigma}_\eta$.
    Top row: (a)~Exact analytic posterior $p(x\mid y)$. (b)~Monte Carlo density estimate from the matched sampler at $r=0.95$. (c)~Monte Carlo density estimate from the unmatched sampler at $r=0.95$; the broader spread along $x_1$ reflects the inflated effective likelihood covariance.
    Bottom row: (d)~KDE of the observed coordinate $\mathbf{H}x=x_1$ at $r=0.95$: the matched sampler (blue) closely follows the exact posterior (black dashed), while the unmatched sampler (red) assigns excess mass to the side columns.
    (e)~JS divergence of the $\mathbf{H}x$ marginal versus normalized split noise $r$. The matched sampler maintains low divergence that decreases with $r$; the residual error at small $r$ is consistent with finite-chain mixing (see panel~f) rather than target bias.
    (f)~Integrated autocorrelation time of $\mathbf{H}x$ versus $r$; both methods mix faster at larger $r$.
    }
    \label{fig:si_gmm}
\end{figure}

\subsection{Additional results for the 2D $\phi^4$ model}
\label{sec:phi4_additional_results}

This appendix provides additional empirical and analytical support for the replica-exchange results in the 2D $\phi^4$ model. We first show that, along the diffusion-time replica ladder, the system undergoes a clear noise-induced transition from the ordered to the disordered regime. We then give two complementary interpretations of this behavior under the fixed-$Q$ construction used in our experiments: after marginalizing the auxiliary field, increasing generative noise weakens the effective long-wavelength inter-site coupling; after marginalizing the clean field, the same noise smooths the local double-well landscape seen by the auxiliary variables for $h=0$.

\subsubsection{Noise-induced phase transitions along the replica ladder}
\label{sec:re_phase_transition}

In the fixed-$Q$ replica-exchange construction described in Appendix~\ref{sec:phi4_details}, all replicas share the same context term anchored at the production time $t_*=t_{\mathrm{prod}}$, while the forward-process noise level varies across the ladder. A convenient replica coordinate is the effective generative noise
\begin{equation}
\tilde{\sigma}_r = \frac{\sigma_r}{\alpha_r},
\end{equation}
which directly controls the amount of local blurring induced by the forward process.

Figure~\ref{fig:re_phase} shows the thermodynamic behavior of the clean field $\phi$ measured across the replica ladder. As $\tilde{\sigma}_r$ increases, the replicas traverse a sharp crossover from an ordered regime with nonzero magnetization to a disordered regime with vanishing order parameter. At the same time, the susceptibility exhibits a pronounced peak, indicating a noise-induced transition along the ladder. Panel~(c) further maps the resulting phase boundary in the $(J,\tilde{\sigma})$ plane. These results show that the high-noise replicas mix rapidly not merely because they are algorithmically noisier, but because they occupy an effectively softened physical ensemble.

\begin{figure}[hbt]
    \centering
    \includegraphics[width=\linewidth]{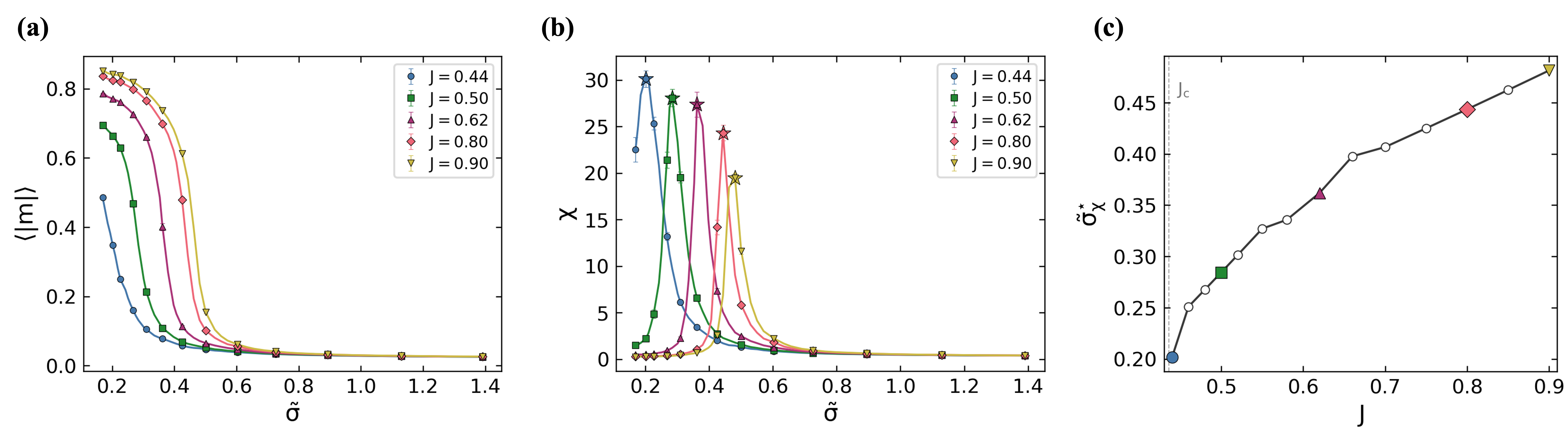}
    \caption{
    \textbf{Noise-induced phase transitions along the replica-exchange ladder.}
    (a) Order parameter $\langle |m_\phi| \rangle$ and (b) susceptibility $\chi_\phi$ measured across diffusion replicas, parameterized by the effective generative noise $\tilde{\sigma}_r=\sigma_r/\alpha_r$. As the noise level increases along the ladder, the replicas cross from the ordered regime into the disordered regime. (c) The extracted phase boundary $\tilde{\sigma}_c^\ast$ as a function of the physical coupling $J$, defining the softening trajectory traversed by the fixed-$Q$ replica-exchange sampler.
    }
    \label{fig:re_phase}
\end{figure}

\subsubsection{Dual physical interpretations of fixed-$Q$ replica exchange}
\label{sec:si_dual_perspective}

The empirical transition in Fig.~\ref{fig:re_phase} admits two complementary analytical interpretations. For the $r$-th replica in the fixed-$Q$ construction, the joint distribution of the clean field $\phi$ and auxiliary noisy field $\psi$ is
\begin{equation}
\Pi_r(\phi,\psi)
\propto
\left[\prod_i p_0(\phi_i)\right]
\mathcal{N}(\psi;\alpha_r\phi,\sigma_r^2 I)
\exp\!\left(-\frac{1}{2}\psi^\top Q_\ast \psi\right),
\label{eq:si_phi4_dual_joint}
\end{equation}
where $p_0(\phi_i)\propto \exp[-(\phi_i^2-1)^2]$ is the on-site double-well prior and the fixed precision matrix $Q_\ast$ is anchored at the production replica $t_\ast=t_{\mathrm{prod}}$. In Fourier space, its eigenvalues are
\begin{equation}
q_k^\ast=\frac{2J\lambda_k}{\alpha_\ast^2-2J\lambda_k\sigma_\ast^2},
\end{equation}
where $\lambda_k$ denotes the eigenvalue of the discrete lattice Laplacian for mode $k$.

\paragraph{Marginalizing the auxiliary field: effective coupling reduction}

We first integrate out the auxiliary field $\psi$ to obtain the effective interaction acting on the clean variables $\phi$. Because both the forward kernel and the context are Gaussian in $\psi$, the resulting quadratic interaction is diagonal in Fourier space. The effective stiffness of mode $k$ in replica $r$ is
\begin{equation}
\kappa_k^{(r)}
=
\alpha_r^2\left(\sigma_r^2 + (q_k^\ast)^{-1}\right)^{-1}.
\end{equation}
Substituting
\begin{equation}
(q_k^\ast)^{-1}
=
\frac{\alpha_\ast^2}{2J\lambda_k}-\sigma_\ast^2,
\end{equation}
gives
\begin{equation}
\kappa_k^{(r)}
=
\frac{2J\lambda_k\,\alpha_r^2}
{\alpha_\ast^2 + 2J\lambda_k(\sigma_r^2-\sigma_\ast^2)}.
\label{eq:si_phi4_kappa}
\end{equation}
For the production replica $r=\ast$, this reduces to
\begin{equation}
\kappa_k^{(\ast)}=2J\lambda_k,
\end{equation}
which exactly recovers the original physical coupling.

For higher-noise replicas with $\sigma_r>\sigma_\ast$, the denominator in Eq.~\eqref{eq:si_phi4_kappa} is strictly larger, implying a mode-dependent screening of the interaction. In the infrared limit $\lambda_k\to 0$ relevant for ordering transitions,
\begin{equation}
\kappa_k^{(r)}
\approx
2J_{\mathrm{eff}}^{\mathrm{IR}}(r)\lambda_k,
\qquad
J_{\mathrm{eff}}^{\mathrm{IR}}(r)
\approx
J\frac{\alpha_r^2}{\alpha_\ast^2}.
\label{eq:si_phi4_jeff}
\end{equation}
Since $\alpha_r^2<\alpha_\ast^2$ for larger diffusion times, the exploring replicas effectively experience a weaker infrared coupling, which can drive them into the disordered regime even when the production replica remains ordered.

\paragraph{Marginalizing the clean field: topological softening}

A complementary interpretation is obtained by integrating out the clean field $\phi$ instead. This yields
\begin{equation}
\Pi_r(\psi)
\propto
\exp\!\left(-\frac{1}{2}\psi^\top Q_\ast \psi\right)
\prod_i
\int d\phi_i\, p_0(\phi_i)\,
\mathcal{N}(\psi_i;\alpha_r\phi_i,\sigma_r^2).
\label{eq:si_phi4_psi_marginal}
\end{equation}
Defining the effective generative noise
\begin{equation}
\tilde{\sigma}_r=\frac{\sigma_r}{\alpha_r},
\end{equation}
the on-site integral becomes
\begin{equation}
\int d\phi_i\, p_0(\phi_i)\,\mathcal{N}(\psi_i;\alpha_r\phi_i,\sigma_r^2)
=
\frac{1}{\alpha_r}
\left(
p_0 * \mathcal{N}(0,\tilde{\sigma}_r^2)
\right)\!\left(\frac{\psi_i}{\alpha_r}\right),
\end{equation}
so the effective local potential for $\psi_i$ is the blurred double-well
\begin{equation}
V_{\mathrm{blur}}^{(r)}(\psi_i)
=
-\log
\left[
\left(
p_0 * \mathcal{N}(0,\tilde{\sigma}_r^2)
\right)\!\left(\frac{\psi_i}{\alpha_r}\right)
\right]
+ \mathrm{const}.
\label{eq:si_phi4_vblur}
\end{equation}
As $\tilde{\sigma}_r$ increases, the local barrier between the two wells decreases and eventually disappears, so the high-noise replicas evolve in a substantially softened landscape.

Figure~\ref{fig:si_re_psi} validates this second viewpoint directly on the auxiliary field. The auxiliary-field order parameter and susceptibility exhibit the same ordering transition seen in the clean field, while the effective barrier height $\Delta V_{\mathrm{blur}}$ decreases rapidly with $\tilde{\sigma}_r$ and eventually vanishes. This confirms that the large-$t$ replicas mix efficiently because they explore a landscape in which local symmetry-breaking barriers have already been smoothed out by the forward corruption.

\begin{figure}[hbt]
    \centering
    \includegraphics[width=\linewidth]{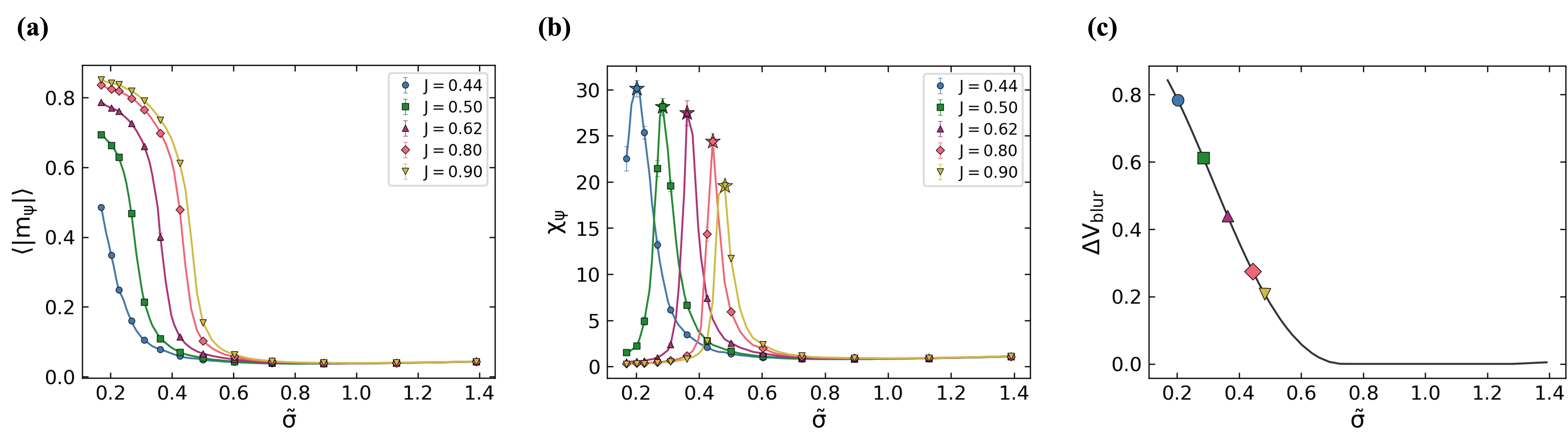}
    \caption{
    \textbf{Thermodynamic behavior and topological softening from the auxiliary-field perspective.}
    (a) Order parameter $\langle |m_\psi| \rangle$ and (b) susceptibility $\chi_\psi$ measured on the auxiliary field $\psi$ across diffusion replicas. As in the clean-field observables, increasing $\tilde{\sigma}_r$ drives the auxiliary field across an ordering transition. (c) Effective local barrier height $\Delta V_{\mathrm{blur}}$ of the blurred on-site potential in Eq.~\eqref{eq:si_phi4_vblur}. The barrier decreases rapidly with generative noise and eventually vanishes, revealing the local softening mechanism that facilitates mixing in the high-noise replicas.
    }
    \label{fig:si_re_psi}
\end{figure}

Taken together, these two views show that the replica coordinate $\tilde{\sigma}_r$ is not only an algorithmic diffusion parameter but also a physically meaningful softening parameter for the fixed-$Q$ ladder. Under this construction, diffusion-time replica exchange plays a role analogous to Hamiltonian parallel tempering: it accelerates sampling by connecting the target ensemble to a family of auxiliary ensembles with weakened long-range ordering tendencies and softened local barriers.
\subsection{Additional results for the alanine dipeptide systems}

\subsubsection{AD--$\text{Na}^+$ system}

To complement the O--O-distance analysis in the main text, Fig.~\ref{fig:si_adsodium} shows the Ramachandran maps of the ion-coupled reference, Direct Diffusion, and GG-PA-RE for priors trained on vacuum data only or with $1\%$ ion-coupled data. We also vary the ion-coupled fraction of the prior training data while keeping its total amount fixed, and compare the resulting fraction of frames with $\phi>0$.

\begin{figure}[hbt]
    \centering
    \includegraphics[width=\linewidth,trim={1bp 0 174bp 0},clip]{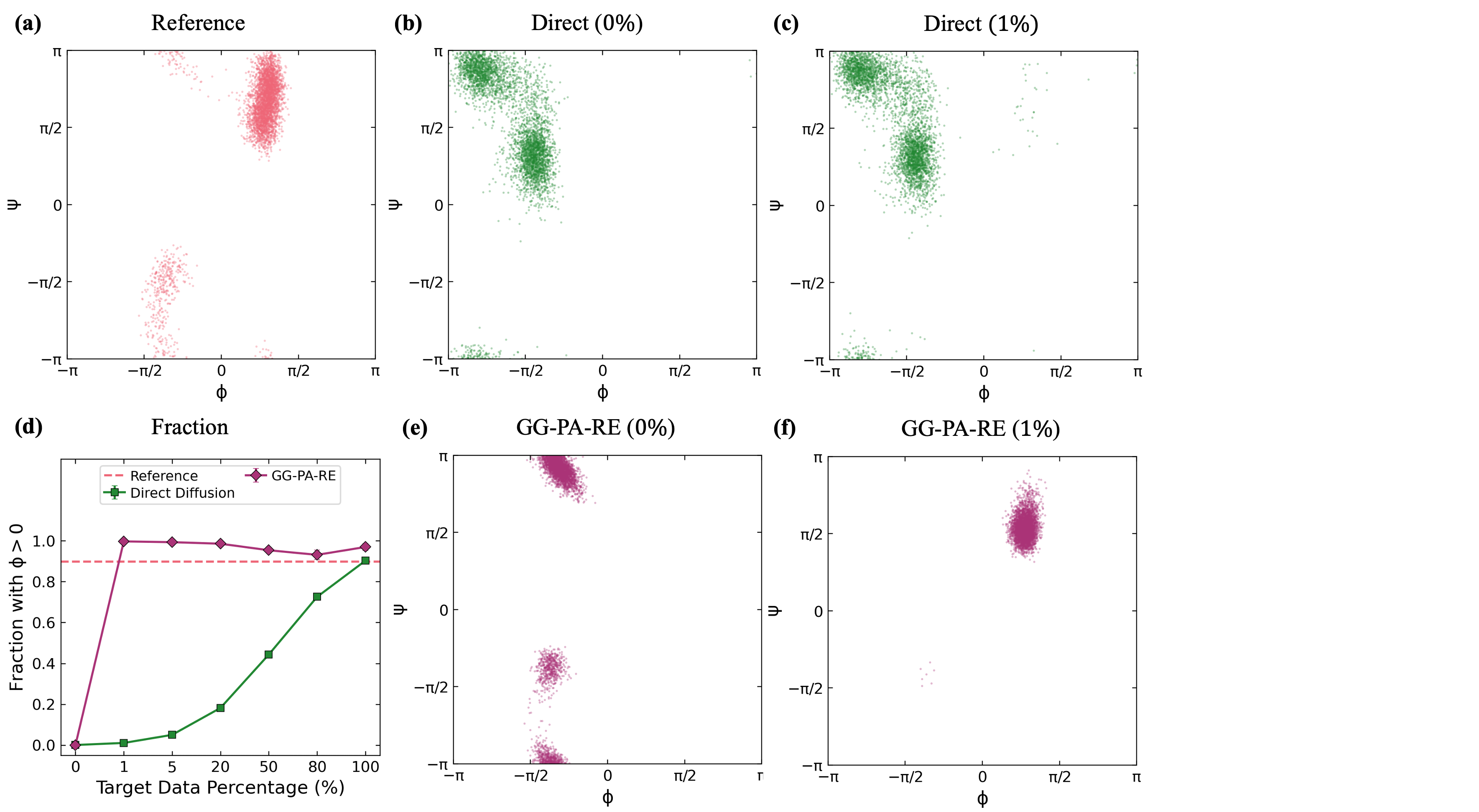}
    \caption{
    \textbf{Ramachandran maps and target-data interpolation for the AD--$\text{Na}^+$ system.}
    (a) Ion-coupled MD reference.
    (b,c) Direct Diffusion with priors trained on $0\%$ and $1\%$ ion-coupled data.
    (d) Fraction of frames with $\phi>0$ versus the ion-coupled percentage of the prior training data. The dashed line marks the reference.
    (e,f) GG-PA-RE with the $0\%$ and $1\%$ priors.
    }
    \label{fig:si_adsodium}
\end{figure}

The ion-coupled reference places $90\%$ of the frames in a $\phi>0$ basin that is absent from vacuum MD. The vacuum-only prior assigns no weight to this region, so neither Direct Diffusion nor GG-PA-RE reaches it (Fig.~\ref{fig:si_adsodium}b,e), even though GG-PA-RE still reproduces the O--O shift. This implies that the composition cannot create states that the prior has never learned. With only $1\%$ ion-coupled data in the prior, however, GG-PA-RE places $99\%$ of the frames in the $\phi>0$ basin (Fig.~\ref{fig:si_adsodium}f), whereas Direct Diffusion remains vacuum-like (Fig.~\ref{fig:si_adsodium}c). GG-PA-RE stays near the reference population at every nonzero percentage, while Direct Diffusion reaches it only when the prior is trained entirely on ion-coupled data (Fig.~\ref{fig:si_adsodium}d). The $\phi>0$ basin of GG-PA-RE is slightly shifted in $\psi$ relative to the reference, and this offset persists even with a prior trained on the reference itself, which points to double counting between the prior and the context. Fixed-$t$ GG-PA rarely crosses between $\phi<0$ and $\phi>0$, so we report GG-PA-RE for this system. 
\subsubsection{AD dimer system}
\label{sec:si_ad_dimer_results}

We analyze the GG-PA-RE production samples for the main dimer results and use fixed-\(t\) GG-PA runs only as supplementary comparisons. The fixed-\(t\) runs use
\[
t\in\{0.1,\,0.15,\,0.25,\,0.4,\,0.6,\,0.8\},
\]
with four chains per diffusion time.

\begin{figure}[hbt]
    \centering
    \includegraphics[width=\linewidth]{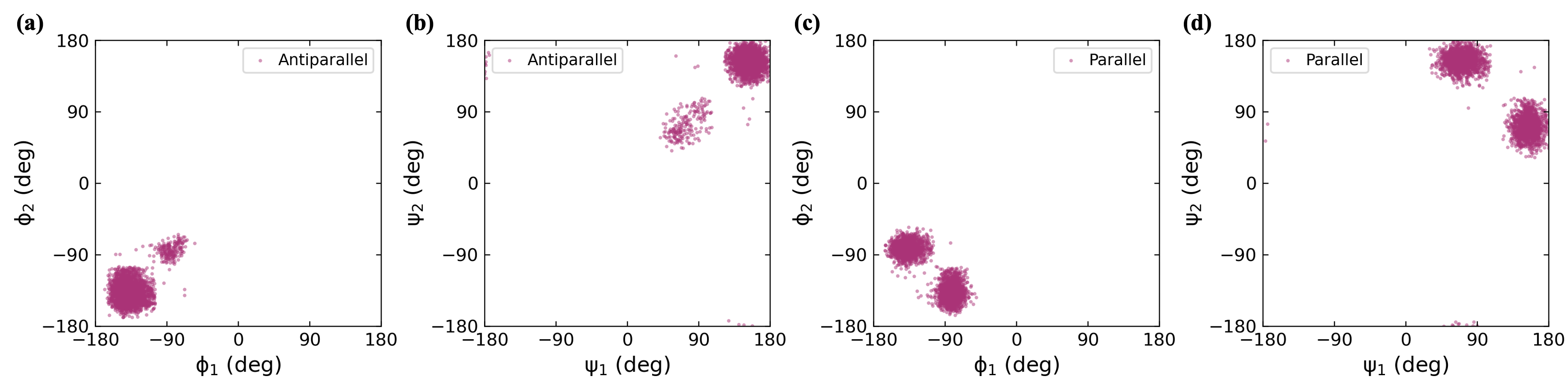}
    \caption{
    \textbf{Ramachandran-space visualization of topology-dependent organization in the AD dimer.}
    All panels show assigned GG-PA-RE frames.
    (a--b) \(\phi_1\) vs.\ \(\phi_2\) and \(\psi_1\) vs.\ \(\psi_2\) for anti-parallel dimers.
    Anti-parallel dimers are dominated by the shared-basin \(LL\) organization.
    (c--d) Corresponding plots for parallel dimers.
    Parallel dimers preferentially occupy the inequivalent \(LR\) and \(RL\) branches.
    }
    \label{fig:si_addimer_rama}
\end{figure}

Fig.~\ref{fig:si_addimer_rama} visualizes the symmetry-broken organization summarized in the main text. Anti-parallel dimers concentrate in the \(LL\) branch, whereas parallel dimers prefer \(LR/RL\).

Following the same coarse-graining, we quantify how fast each sampler switches between dimer states. Each frame is mapped to the last visited of two core states, either the two topologies $A$ and $P$ or the two parallel branches $P_{LR}$ and $P_{RL}$. The switching time is the sampled MD time divided by the number of changes of core, evaluated at a common resolution of $10$~ps. Table~\ref{tab:si_ad_transitions} reports the mean first-passage times between core states.

To quantify sampling accuracy, we compute the topology-resolved pooled Jensen--Shannon divergence over the absolute torsional difference $|\Delta\psi|=|\psi_1-\psi_2|$,
\begin{equation}
JS_{\psi}
=
JS_{\mathrm{anti}}(|\Delta\psi|)
+
JS_{\mathrm{para}}(|\Delta\psi|),
\end{equation}
over assigned frames, and the JS divergence $JS_{\phi\psi}$ of the Ramachandran histogram of both copies over all frames. Both use $10^\circ$ bins and are computed against the MD reference.

\begin{figure}[hbt]
    \centering
    \includegraphics[width=\linewidth]{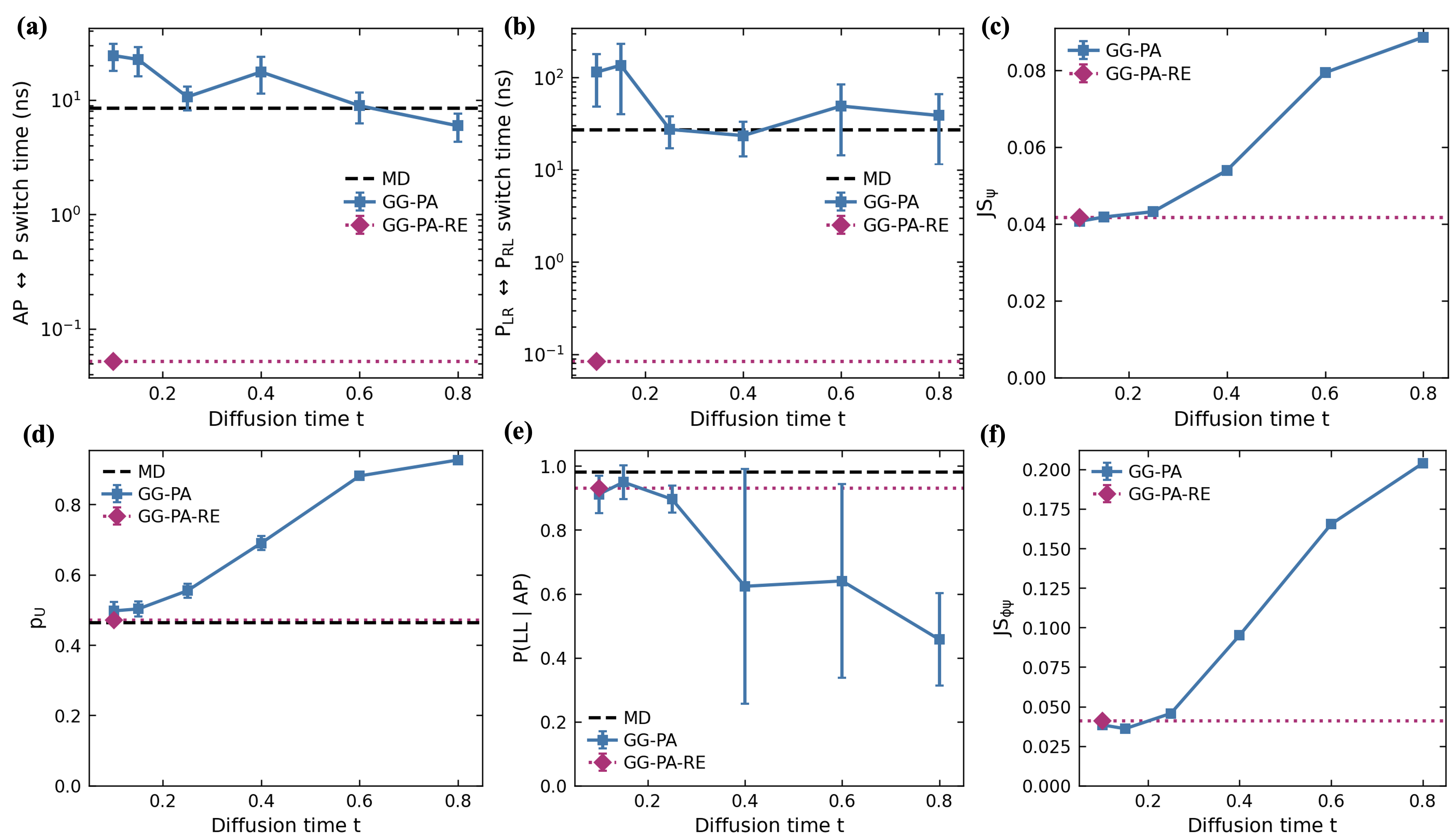}
    \caption{
    \textbf{Fixed-$t$ GG-PA and GG-PA-RE in the AD dimer system.}
    Fixed-$t$ runs are plotted against their diffusion time. GG-PA-RE is shown at its production replica with $t=0.1$, with dotted lines for comparison and dashed lines for the MD reference. JS denotes the Jensen--Shannon divergence from the MD reference.
    (a) Switching time between anti-parallel and parallel topologies.
    (b) Switching time between the $P_{LR}$ and $P_{RL}$ branches.
    (c) $JS_{\psi}$, for the torsion difference $|\psi_1-\psi_2|$ between the two copies.
    (d) Unassigned fraction $p_U$.
    (e) Occupancy of $LL$ within the anti-parallel topology.
    (f) $JS_{\phi\psi}$, for the Ramachandran histogram of both copies.
    }
    \label{fig:si_addimer_iat}
\end{figure}

The results are shown in Fig.~\ref{fig:si_addimer_iat}. Raising $t$ in fixed-$t$ runs degrades all accuracy measures without making switching faster than MD, whereas GG-PA-RE keeps the accuracy of the $t=0.1$ rung and switches two to three orders of magnitude faster than both (Table~\ref{tab:si_ad_transitions}). On the same device, one topology switch costs about $17$ CPU seconds with GG-PA-RE including all replicas, about $120$~s with MD, and $700$ to $1600$~s with fixed-$t$ GG-PA. These switching times measure sampler efficiency rather than the physical kinetics of the dimer.

\begin{table}[hbt]
\centering
\caption{\textbf{Mean first-passage times between core states in the AD dimer.} Times are given in ns of sampled MD time at $10$~ps resolution.}
\label{tab:si_ad_transitions}
\begin{tabular}{l r@{${}\pm{}$}l r@{${}\pm{}$}l}
\toprule
Transition & \multicolumn{2}{c}{MD} & \multicolumn{2}{c}{GG-PA-RE} \\
\midrule
$A \to P$ & $8.81$ & $0.09$ & $0.038$ & $0.002$ \\
$P \to A$ & $7.81$ & $0.55$ & $0.066$ & $0.003$ \\
$A_{LL} \to P_{LR}$ & $21.80$ & $1.00$ & $0.087$ & $0.009$ \\
$A_{LL} \to P_{RL}$ & $24.20$ & $1.30$ & $0.072$ & $0.005$ \\
$P_{LR} \to A_{LL}$ & $9.25$ & $1.00$ & $0.082$ & $0.005$ \\
$P_{RL} \to A_{LL}$ & $8.93$ & $0.15$ & $0.076$ & $0.005$ \\
$P_{LR} \to P_{RL}$ & $25.90$ & $1.50$ & $0.071$ & $0.005$ \\
$P_{RL} \to P_{LR}$ & $25.80$ & $2.30$ & $0.085$ & $0.008$ \\
$A_{LL} \to A_{RR}$ & $35.90$ & $8.30$ & $0.840$ & $0.090$ \\
$A_{RR} \to A_{LL}$ & $6.55$ & $0.73$ & $0.110$ & $0.010$ \\
\bottomrule
\end{tabular}
\end{table}

\paragraph{Residual unassigned category.}
Fig.~\ref{fig:si_addimer_u_breakdown} decomposes all frames by the first failed assignment step of Appendix~\ref{subsec:ad_dimer_system}. GG-PA and GG-PA-RE closely match both \(p_U\) and its composition in the reference.

\begin{figure}[hbt]
    \centering
    \includegraphics[width=0.85\linewidth]{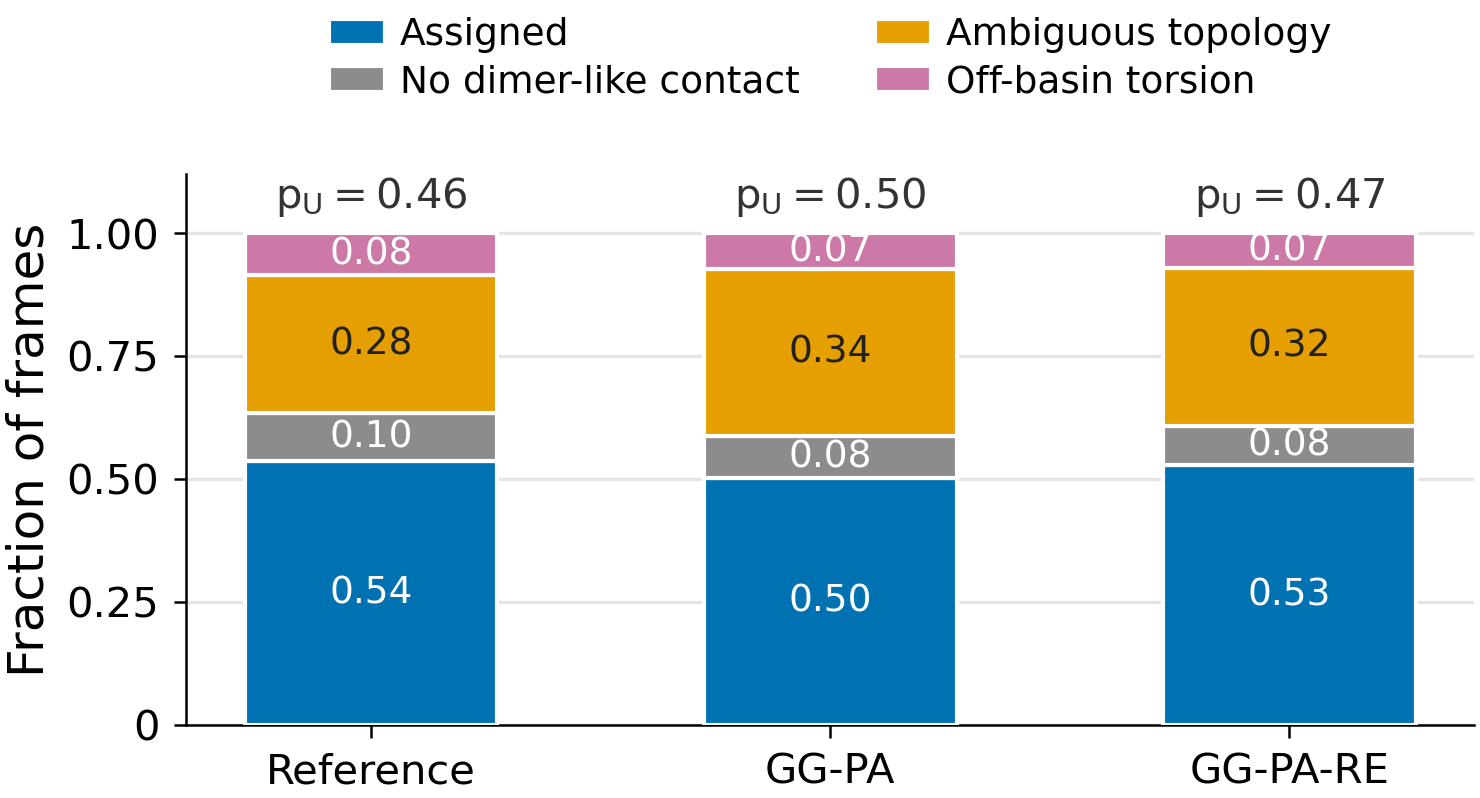}
    \caption{
    \textbf{Decomposition of the residual unassigned category in the AD dimer.}
    Bars show the fraction of frames assigned to the main topology--basin categories or placed in the residual category \(U\), decomposed by the first failed assignment step.
    Assigned frames satisfy all three criteria: \(N_{\mathrm{HB}}\geq 2\) with reciprocal contact, \( |\cos\theta|\geq 0.8 \), and both monomers in the predefined \(L/R\) torsional basins.
    No-dimer-like-contact frames fail the hydrogen-bond/reciprocal-contact diagnostic.
    Ambiguous-topology frames pass contact criteria but have \( |\cos\theta|<0.8 \).
    Off-basin-torsion frames pass contact and orientation criteria but have at least one monomer outside the \(L/R\) basins.
    The scalar \(p_U=P(U)\) is shown above each bar.
    }
    \label{fig:si_addimer_u_breakdown}
\end{figure}

The main remaining difference from MD is the topology balance. Among assigned frames, GG-PA-RE gives anti-parallel and parallel weights of $0.35$ and $0.65$, against $0.50$ and $0.50$ for MD, possibly because residual double counting between the prior and the context slightly favors the parallel topology.

\end{document}